\documentclass{article}

\usepackage[preprint]{neurips_2026}

\usepackage[utf8]{inputenc} 
\usepackage[T1]{fontenc}    
\usepackage{hyperref}       
\usepackage{url}            
\usepackage{booktabs}       
\usepackage{amsfonts}       
\usepackage{nicefrac}       
\usepackage{microtype}      
\usepackage{xcolor}         
\usepackage{graphicx}      
\usepackage{amsmath}        
\usepackage{amssymb}        
\usepackage{amsthm}         
\usepackage{algorithm}      
\usepackage{algorithmic}    
\usepackage{multirow}       

\usepackage{subcaption}
\usepackage{pdflscape}  

\usepackage{wrapfig}
\usepackage{subcaption}

\newtheorem{proposition}{Proposition}
\newtheorem{lemma}{Lemma}

\title{Optimal Recourse Summaries via Bi-Objective Decision Tree Learning}

\author{
  Ioannis~Chatzis \\
  \texttt{johnco.chatzis@gmail.com} \\
  \And
  Jason~Liartis \\
  \texttt{jliartis@ails.ece.ntua.gr} \\
  \AND
  Athanasios Voulodimos \\
  \texttt{thanosv@mail.ntua.gr} \\
  \And
  Giorgos Stamou \\
  \texttt{gstam@cs.ntua.gr} \\
  \And
  \\
  Artificial Intelligence and Learning Systems Laboratory\\
  National Technical University of Athens, Greece
}

\begin{document}

\maketitle

\begin{abstract}
  Actionable Recourse provides individuals with actions they can take to change an unfavorable classifier outcome. While useful at the instance level, it is ill-suited for global auditing and bias detection, since aggregating local actions is costly and often inconsistent. Recourse Summaries address this limitation by partitioning the population and assigning one shared action per subgroup, enabling comparison across subgroups. Designing summaries involves a fundamental trade-off between recourse effectiveness and recourse cost, which existing methods do not adequately address. We introduce Summaries of Optimal and Global Actionable Recourse (SOGAR), which formulates recourse summary learning as an optimal decision tree learning problem and finds the Pareto front -- the complete set of solutions where improving one objective necessarily worsens the other. SOGAR enables post-hoc selection of the desired trade-off without retraining. Using shallow axis-parallel decision trees and sparse leaf actions, SOGAR produces stable, low-cost, and effective recourse summaries that outperform existing approaches across effectiveness and cost metrics.
\end{abstract}

\section{Introduction}
\label{sec:intro}

Automated decision-making systems based on opaque machine learning models are increasingly being applied in critical domains such as credit scoring, hiring, education, and healthcare decisions.
In such high-risk settings, high predictive accuracy is insufficient for stakeholders, as the inspection of the system's decisions is not only useful but often a legal requirement as well, e.g., by Article 86 of the EU AI Act.\footnote{\url{https://artificialintelligenceact.eu/article/86/}}
To address this, there are efforts both on the fronts of inherent interpretability \cite{rudin2019stop} and post-hoc interpretability \cite{ribeiro2016should,Lundberg_Lee_2017}.

A very popular post-hoc method is \emph{counterfactual explanations} which specify the changes to the input features that would produce a different classifier output \cite{wachter2018counterfactualexplanationsopeningblack}.
Counterfactual explanations are very interpretable and often preferred by humans as they mimic their thinking process  \cite{guidotti2018survey}.
Ustun et al. reframe the problem around an individual who needs to take actions to change the unfavorable decision of an automated system, terming this setting \emph{actionable recourse} \cite{Ustun_Spangher_Liu_2019}.

\emph{Global counterfactual explanations} and \emph{recourse summaries} extend these explanations to provide insight for an entire population.
Global counterfactual explanations try to identify a small set of low-cost actions that effectively change the classifier's output with respect to a reference population \cite{Ley_Mishra_Magazzeni_2023, kavouras2025glanceglobalactionsnutshell}.
Recourse summaries further identify population subgroups and assign one action to each subgroup, which is especially helpful for identifying biases by comparing the optimal action for each subgroup \cite{rawal2020beyond, cet_Kanamori_Takagi_Kobayashi_Ike}.
Both approaches try to balance a set of objectives that often include the following: \textbf{(i)} Coverage, the portion of the population that is assigned with actions for recourse, \textbf{(ii) }Recourse cost, the estimated effort of enacting the assigned recourse, \textbf{(iii) }Recourse loss, the portion of the population whose recourse successfully flips the classifier's output. Also referred to as ineffectiveness or incorrectness, \textbf{(iv)} Action sparsity, the number of input features affected by the assigned recourse, \textbf{(v)} Subgroup count, \textbf{(vi)} Subgroup description length, how many features are used to specify each subgroup.
Not all of these objectives are sufficiently addressed by existing works.
Some provide limited coverage \cite{rawal2020beyond}, some prioritize loss over cost \cite{kavouras2025glanceglobalactionsnutshell}, others scalarize the objectives with arbitrary coefficients \cite{rawal2020beyond, cet_Kanamori_Takagi_Kobayashi_Ike}, and none solve their provided formulation to a \emph{global optimum}.

\begin{wrapfigure}{r}{0.5\textwidth}
    \centering
    \includegraphics[width=\linewidth]{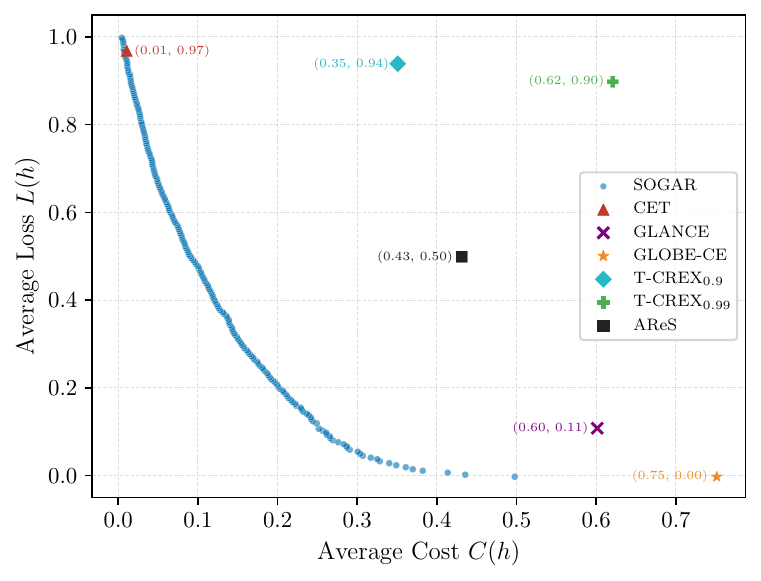}
        \caption{SOGAR vs.\ baselines on Employee Attrition, LightGBM classifier, fold 1 of the 10-fold cross validation.
                 Closest to $(0,0)$ is better. Blue dots show 
                 the full SOGAR trade-off; other methods produce 
                 a single solution each.}
        \label{fig:pareto_intro}
\end{wrapfigure}

To address this, we introduce \emph{Summaries of Optimal and Global Actionable Recourse (SOGAR)}, which formulates recourse summary learning as a bi-objective optimal decision tree problem, the objectives being the recourse cost and loss.
Each summary is represented as a shallow, axis-parallel decision tree that partitions the affected population, with one action assigned per leaf.
We produce a small number of groups with simple descriptions and sparse actions by imposing thresholds on the number of leaves, tree depth, and action sparsity.
Under the separable conditions of the STreeD framework \cite{streed_VAN_2023}, this yields solutions that guarantee global optimality over the specified objectives and constraints.
Rather than committing to one single scalarization of cost and effectiveness, SOGAR produces the full Pareto front of non-dominated summaries for a fixed maximum depth and number of leaves, exposing the complete trade-off between the objectives.
Since finding optimal decision trees is NP-hard \cite{hyafil1976optimal}, it also offers early termination, returning the best solutions found within the allocated time budget.

Our contributions as summarized are the following:
\begin{itemize}
    \item We define SOGAR as a bi-objective optimal decision tree learning problem for actionable recourse summaries that jointly optimizes recourse effectiveness and action cost under tree size and action sparsity constraints.
    \item We cast SOGAR as an optimization task within the dynamic programming framework STreeD, prove it is exactly solvable, compute the full non-dominated Pareto front, and further accelerate the leaf-action evaluation by introducing CPU and GPU parallelization.
    \item We evaluate SOGAR against recourse-summary and global counterfactual explanation baselines on benchmark tabular datasets and show that the exact optimization recovers additional non-dominated summaries, thereby improving the cost-effectiveness trade-off.  
    \item We validate SOGAR's auditing capabilities by recovering the gender-based bias of the Adult dataset \cite{adult_2} in the Pareto front.
\end{itemize}

The remainder of the paper is organized as follows. In Section \ref{subsec:related_work} we present the related literature. In Section \ref{sec:problem_statement} we provide the preliminaries and formalize recourse summaries and the SOGAR objective.
In Section \ref{sec:proposed_method} we present the end-to-end formulation of our framework.
Finally, in Section \ref{sec:experiments} we compare our method against 5 other baselines, across 4 datasets of different sizes and discuss the results.

\section{Related Work}
\label{subsec:related_work}

\paragraph{Counterfactual Explanations and Algorithmic Recourse.}
Counterfactual Explanations (CE) and Algorithmic Recourse (AR) are two similar explainability techniques, that given a classifier and an input instance, find an alternative instance that changes the classifier's output.
CEs minimize a notion of distance between the original and the alternative instances \cite{wachter2018counterfactualexplanationsopeningblack,Chou_Moreira_Bruza_Ouyang_Jorge_2022}.
AR, being more human-centric, frames the problem as providing individuals who have been treated unfavorably by an automated decision-making system with actions they can take to alter the decision, and minimizes the burden of enacting the recourse
\cite{Ustun_Spangher_Liu_2019, Karimi_Scholkopf_Valera_2021}.
AR typically imposes constraints on the proposed actions, such as actionability, feasibility, proximity and sparsity of edits.

\paragraph{Global Counterfactual Explanations and Recourse Summaries.}
Global Counterfactual Explanations and Recourse Summaries are directly situated in the global setting, identifying a small set of edits or actions that minimize average recourse cost and loss --- the frequency with which the assigned recourse fails to flip the classifier's label.
The first work along those lines was AReS, which uses itemset mining to find two-level population subgroups and assigns one action to each subgroup while optimizing for recourse coverage, cost, loss, and interpretability \cite{rawal2020beyond}.
AReS has two major limitations: high running times and inconsistency --- some individuals may be assigned multiple actions while others receive none. T-CREx adopts the tree structure to assign counterfactuals to individuals, but prioritizes solution sparsity over effectiveness and does not impose actionability constraints \cite{bewley2024counterfactual}. 
GLOBE-CE takes a different approach: rather than finding specific actions that flip the classifier's output, it identifies fixed directions that can be added to inputs with variable magnitude \cite{Ley_Mishra_Magazzeni_2023}.
GLANCE begins by clustering the population and generating a set of candidate actions for each cluster \cite{kavouras2025glanceglobalactionsnutshell}.
It then merges clusters that are similar in feature or action space and selects the optimal action for each, prioritizing coverage, yielding a small final set of counterfactual actions.
CET is the work closest to ours \cite{cet_Kanamori_Takagi_Kobayashi_Ike}.
It partitions the population into subgroups using a decision tree structure, assigning one action to each leaf.
This ensures that an individual is assigned exactly one action.
CET uses stochastic local search to find good tree splits and  minimizes a linear combination of cost and ineffectiveness, configured by a hyperparameter $\gamma$, to assign actions to leaves.
Like AReS, CET has been criticized for excessive running times, which we confirm in our experiments in Section \ref{sec:experiments}.

SOGAR, our method, shares CET's tree-based format but differs in key respects.
Instead of stochastic local search, we find globally optimal solutions using dynamic programming.
Additionally, we do not linearize the objective; instead, we identify all Pareto-optimal solutions that improve either cost or ineffectiveness, allowing users to select the solution that best satisfies their needs.
This approach also eliminates the need to recompute solutions for different values of $\gamma$, if the initial solution proves unsatisfactory.
Although finding global optima is NP-hard, the STreeD framework \cite{streed_VAN_2023} proves much faster in practice through efficient caching and pruning, which we enhance with additional action-cost caches and a GPU implementation.

\paragraph{Optimal Decision Trees.}
Decision Tree learning is NP-hard \cite{hyafil1976optimal} and is most often approached with a greedy algorithm such as CART \cite{breiman1984classification} or C4.5 \cite{quinlan1993c4}.
Modern advancements in hardware and optimization algorithms have made finding the global optimum computationally feasible.
The first successful method utilized itemset lattices \cite{nijssen2007mining}, which was further developed \cite{aglin2020learning} and other approaches emerged using Mathematical Programming \cite{bertsimas2017optimal, aghaei2021strong}, Constraint Programming \cite{narodytska2018learning}, Branch and Bound search \cite{hu2019optimal, mctavish2022fast} and Dynamic Programming \cite{demirovic2022murtree, streed_VAN_2023}.
Most works are concerned with optimizing accuracy, but others allow optimization of more complex metrics such as F1-score \cite{Lin_Zhong_Hu_Rudin_Seltzer_2022} or other tasks altogether such as survival analysis \cite{zhang2024optimal} or clustering \cite{bertsimas2018interpretable}.
The STreeD framework in particular, introduced in \cite{streed_VAN_2023}, presents a set of necessary and sufficient conditions, which, when satisfied, allow for any task to be solved by their algorithm, which utilizes dynamic programming with caching and pruning.
We prove that finding recourse summaries is such a task, and we integrate it into the STreeD framework.
To our knowledge, SOGAR is the first work in the field of optimal decision trees to tackle actionable recourse or counterfactual explanations.

\section{Problem Setting}
\label{sec:problem_statement}

Here we introduce our notation and formulation of recourse summaries.

Throughout this paper, we assume that we are working in a classification setting, where a fixed classifier $f$ assigns labels from $\{0, 1\}$ to instances drawn from a set $\mathcal{X}$.
We assume that 1 is the positive, desired label and 0 is the negative, undesired label.
In multiclass settings we can simply designate which labels are desired and map those to 1 and the rest to 0.
We also assume that instances are $n$-dimensional vectors of any combination of real, categorical and binary features.
We denote the $j$-th feature of an instance by $x_j$.
We also assume that we have an action set $\mathcal{A}$ whose elements $a$ transform an instance $x$ to a new instance $a(x) \in \mathcal{X}$.
For example, $f$ could be a loan approval model, with ``approved'' being the desired class and ``rejected'' being the undesired class.
$\mathcal{X}$ could consist of features such as the applicant's age, income and requested loan amount, while an action could be ``Increase your income by \$5,000 and reduce the requested loan amount by \$20,000''.
Finally, for computational purposes, we assume that we have a binarization function $\phi: \mathcal{X} \rightarrow \{0, 1\}^{|\mathcal{F}|}$ where $\mathcal{F}$ is a set of binary features derived from thresholds on numerical features and encoding on categorical features.

\paragraph{Individualized Recourse.}
Let $c$ denote a function $c: \langle \mathcal{X}, \mathcal{A} \rangle \rightarrow \mathbb{R}_{\geq 0}$ that assigns a cost to an action $a$ for an instance $x$.
Individualized recourse refers to finding the minimal cost action that maps an instance $x$ to a positively classified instance, i.e.,
\begin{equation}
    \arg\min_{a \in \mathcal{A}} c(a,x), ~s.t.~ f(a(x)) = 1
\end{equation}

\paragraph{Recourse Summaries.}
Given a population of instances along with the labels assigned to them by the classifier $D = \{ \langle x^{(i)}, f(x^{(i)})\rangle\}_{i=0}^N$, the affected instances is the subset of instances assigned to the undesired class, $D_0 = \{ x^{(i)}: \langle x^{(i)}, f(x^{(i)}) \rangle \in D, ~f(x^{(i)}) = 0 \}$.
Assigning one action to the entire affected population, could lead to arbitrarily high costs when $D_0$ contains points far from the decision boundaries, as shown in Proposition 1 of \cite{cet_Kanamori_Takagi_Kobayashi_Ike}.

In order to keep recourse costs low, while assigning only a few actions to the entire population, \emph{recourse summaries}, as introduced in \cite{rawal2020beyond}, define population subgroups $X_k \subseteq \mathcal{X}, k=1,\dots, K$ and assign one action to each subgroup.
The objective is additionally relaxed by allowing an action to be ineffective for some instances of the affected population, but we measure the count of those instances as its \emph{recourse loss}:
\begin{align}
\label{eq:single_action_loss}
    l(a,x) &= \mathbb{I}[f(a(x)) = 0]\\
\label{eq:leaf_loss}
    l(a,X_k) &= \sum_{x \in X_k} l(a,x)
\end{align}
Where $\mathbb{I}$ is the indicator function.

Using the subgroups $X_k$ and the actions assigned to them $a_k$ we can treat the recourse summary as a mapping $h$ from instances to actions, such that $h(x) = a_k$ if $x \in X_k$.
Therefore, we can define the cost and the loss functions at the level of the recourse summary by aggregating over $D_0$:
\begin{align}
    C(h) &= \sum_{x \in D_0} c(h(x), x)
    \label{eq:global_cost}\\
    L(h) &= \sum_{x \in D_0} l(h(x), x)
    \label{eq:global_loss}
\end{align}

Increasing the number of subgroups can only improve $C$ and $L$, as is shown in Proposition 2 of \cite{cet_Kanamori_Takagi_Kobayashi_Ike}.
Despite this, we need the number of subgroups to be relatively small, otherwise it becomes hard to draw global insights.

For a fixed set of subgroups, there is a fundamental trade-off between the cost and the loss of the actions assigned to each subgroup.
To decrease the loss, we need to reduce the number of instances for which the assigned action is ineffective, thus the pool of available actions becomes smaller and the cost of the chosen action can only increase or stay the same.
As discussed in Section \ref{subsec:related_work}, prior work either solves for a linear combination of the objectives or enforces a threshold on one of them and optimizes the other and do not provide any optimality guarantees for their solutions.
SOGAR instead computes the full Pareto front of globally optimal solutions.
This avoids arbitrary thresholds and scalarization coefficients on the objectives and allows the end-user to select their ideal trade-off between cost and recourse loss.
This also avoids the issue of having to re-run the process for computing recourse summaries multiple times if the initial selection of objective coefficients and thresholds results in unsatisfactory solutions.
SOGAR is the first method to provide this guarantee for recourse summaries.

\paragraph{Recourse Summary Trees.}
A \emph{recourse summary tree} is a decision tree that defines a mapping $\tau: \mathcal{X} \rightarrow \mathcal{A}$.
Each internal node $v$ performs a test of the form $[x_{j(v)} \leq b_v]$.
If an instance passes the test, it proceeds recursively to the right child of $v$ and to the left child if it fails.
These splits induce a partition of $\mathcal{X}$ into a finite collection of disjoint subgroups $\{ X_k \}_{k \in \mathcal{L}}$, where $\mathcal{L}$ indexes the leaf nodes, and thus $\tau$ serves as a recourse summary.
Each leaf is labeled with an action $a_k \in \mathcal{A}$ and the induced decision rule is $\tau (x) = a_k$ if $x \in X_k$.
Trees with simple tests on features result in very simple subgroup definitions and by controlling their maximum leaf count and maximum depth we can control the subgroup count and their description length.

Given the pair of minimization objectives $(C, L)$, a solution $\tau$ is said to strictly dominate $\tau'$ if $C(\tau) \leq C(\tau')$, $L(\tau) \leq L(\tau')$ and at least one of the inequalities is strict, denoted $\tau \succ \tau'$.
Given a set of candidate recourse summary trees $\mathcal{T}$, such as all trees with a maximum depth of 3, the Pareto front $P$ of $\mathcal{T}$ is the set of non-dominated solutions:
\begin{equation}
    P = \{ \tau \in \mathcal{T} ~|~ \nexists \tau' \in \mathcal{T} ~(\tau' \succ \tau) \}
    \label{eq:pareto}
\end{equation}
Each point on the Pareto front represents a summary where improving cost would require sacrificing effectiveness, and vice versa.

\begin{proposition}[Optimality] \label{prop:optimality}
For any maximum depth $d$, the complete Pareto front of recourse summary trees, all trees that are not dominated by another tree of maximum depth $d$ in both cost $C$ and recourse loss $L$, can be computed exactly via dynamic programming.
\end{proposition}

The proof (Appendix~\ref{app:proof}) shows that recourse summary tree optimization is a \emph{separable} optimization task in the sense of \cite{streed_VAN_2023}, enabling the use of their dynamic programming framework, named STreeD, with depth and leaf count constraints for interpretability, and caching and pruning for efficient computation. STreeD uses the binarized features $\mathcal{F}$ for defining splits, but for accurately computing the objectives, we use the original features $\mathcal{X}$.
Thus, our method operates on the binarized dataset $D_B = \{ \langle \phi(x), x \rangle : x \in D_0 \}$ that contains pairs of (binarized, original) instances.
\section{Proposed Method: SOGAR}
\label{sec:proposed_method}

\begin{figure}[t!]
    \centering
    \newsavebox{\pipelinebox}
    \savebox{\pipelinebox}{\includegraphics[width=\linewidth]{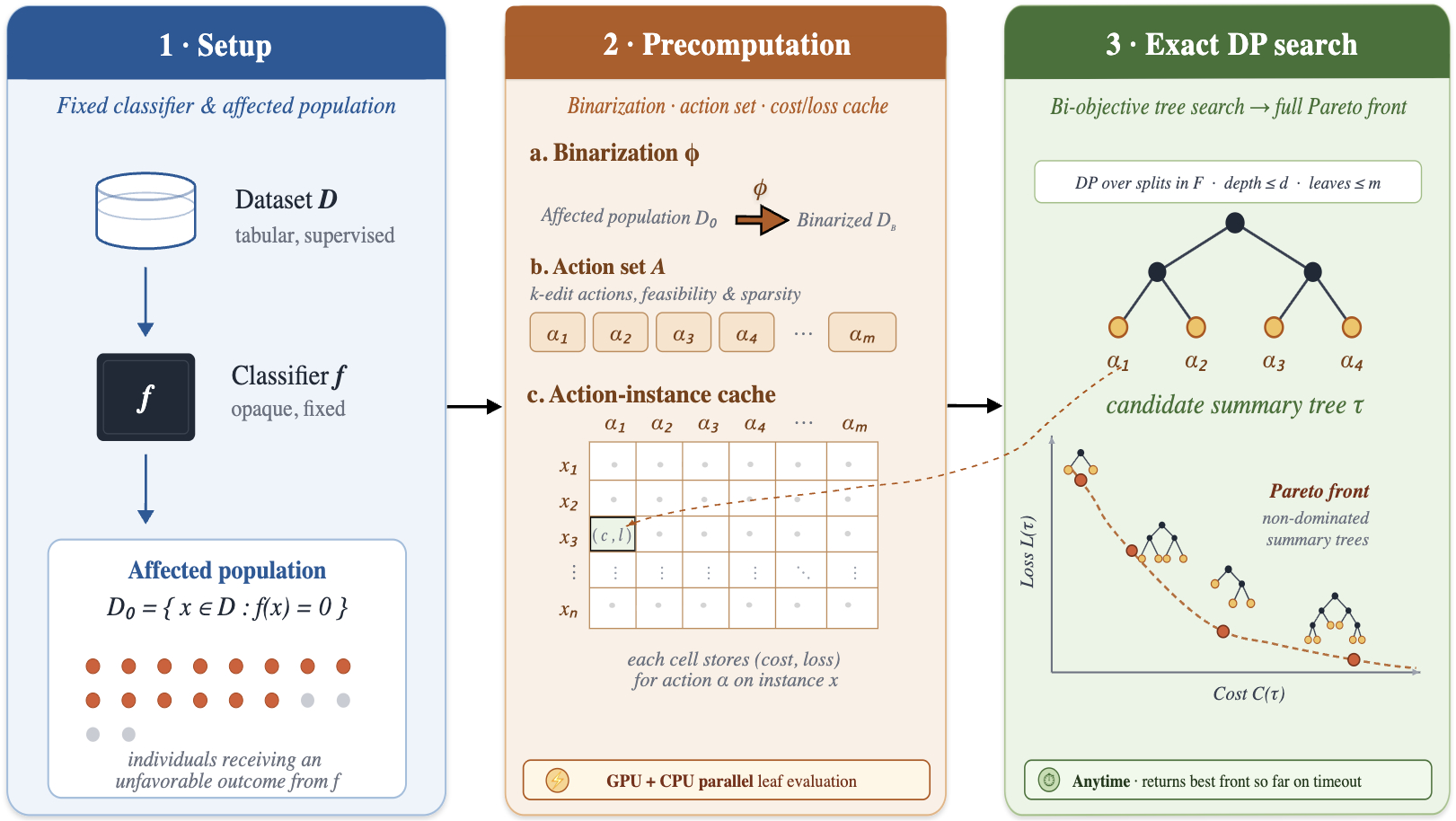}}
    \resizebox{\linewidth}{0.9\ht\pipelinebox}{\usebox{\pipelinebox}}
    
    \caption{The figure depicts the end-to-end pipeline of SOGAR. From left to right, the process starts by fixing a classifier \(f\), and retrieving the affected population that requires recourse, and immediately proceeding to the binarization process through \(\phi\). Then, the set of feasible shared actions \(\mathcal{A}\) is created across the entire population, which is used to compute the cache of cost and loss, for every action applied on every instance. Finally, the optimization task is solved using the dynamic programming framework of STreeD and the complete Pareto front of solutions that improve either loss or cost is computed.
}
    \label{fig:sogar_pipeline}
\medskip

\end{figure}

In this section, we present our method, SOGAR, as an end-to-end pipeline, under the formalization of Section \ref{sec:problem_statement}
which is visualized in Figure \ref{fig:sogar_pipeline}.

\paragraph{Action Set.}
SOGAR defines a single-edit action as one of the following operations.
Binary variables: Flipping the value. Categorical variables: Changing from one category to another.
Numerical variables: Increasing or decreasing the value by $n$ bins.
Single-edit actions can be pooled together to form $k$-edit actions.
The action set $\mathcal{A}$ contains all actions of up to $k$ edits, with the default value of $k$ being 3.
The user can also impose actionability constraints by specifying features for which no actions will be generated.

\paragraph{Cost \& Loss functions.}
As cost the function $c$, SOGAR uses the \emph{Maximum Percentile Shift (MPS)} function, similar to \cite{cet_Kanamori_Takagi_Kobayashi_Ike, Ustun_Spangher_Liu_2019}:
\begin{equation}
\label{eq:mps}
    c_{MPS}(a,x) = \max_{j} |Q_j(a(x)_j)-Q_j(x_j)|,
\end{equation}
where \(Q_j\) denotes the cumulative distribution function (CDF) of feature \(j\).
The cost function was chosen over a simple norm-based distance function because it is scale-invariant \cite{cet_Kanamori_Takagi_Kobayashi_Ike} and yields a realistic depiction of the effort required to enact the action by quantifying the difficulty of moving within a certain distribution.
Recourse failure is captured by the loss function defined in Equation \eqref{eq:single_action_loss}.

\paragraph{Cache computation.}
The action set \(\mathcal{A}\) is used to create the cache, where the cost and loss are precomputed for every action applied on every affected instance. Directly evaluating \(c_{MPS}\) and the classifier output of every transformed input \(a(x)\) inside tree search is computationally prohibitive because of the combinatorial nature of the problem, where the same instance-action pairs are reused across many candidate leaves. Thus, the \(O(1)\) look-ups of cached cost-loss pairs preserve efficiency.

\paragraph{Binarization.}
We binarize features to obtain the feature set \(\mathcal{F}\), which is required by the STreeD framework to search over all possible variable splits. Categorical features are one-hot encoded
and continuous features or integer features are discretized using equal-width bins. After the tree search algorithm terminates, the binary splits can be back-translated into simple, interpretable equality and threshold tests on the original variables.

\paragraph{Optimization.}
Using the cached costs and losses, we can formulate recourse summary learning as an optimization task in the STreeD framework.
The objective function of the optimization task is the cost-loss pair $C(\tau), L(\tau)$ and the dominance operator $\succ$ is element-wise comparison as outlined in Section \ref{sec:problem_statement}.
We also use constraints on the maximum tree depth $d$, and maximum leaf count $m$, to keep the summaries small and interpretable.
A detailed formulation of the task is also given in Appendix \ref{app:proof}.
STreeD proceeds to solve this optimization problem, returning the full Pareto front of non-dominated solutions, as defined in Equation \eqref{eq:pareto}.
For a given maximum depth $d$, STreeD performs bottom-up dynamic programming, recursing on $d$.
It utilizes caching of partial solutions and strong bounds to prune the search space, dramatically reducing the number of subproblems evaluated.
We further improve the scalability of the algorithm by implementing multi-thread CPU and GPU acceleration as outlined in the following sections.
Even though the nature of the task is combinatorial, the caching and pruning of STreeD, along with our cost-loss cache and parallelization, allows the problem to be solved to global optimality within a reasonable time interval.
Finally, a timeout can be imposed, after which STreeD returns all non dominated solutions it has found thus far, without guaranteeing that the are contained in the true Pareto front.
Our experiments in Appendix \ref{app:timeouts} show that, even with a limited computational budget, SOGAR can recover near-optimal solutions.

\paragraph{Leaf Acceleration.}
Every time a subproblem is solved, for each leaf, STreeD needs to evaluate each action on the individuals reaching that leaf.
Aggregating the cost and loss of these individuals for some action is an embarassingly parallelizable task.
We speed up the solver by implementing multi-thread CPU and GPU computation of these aggregates, significantly reducing running times.
We provide a comparison of running times with and without GPU acceleration in Appendix \ref{app:hardware_acceleration}.

\section{Experiments}
\label{sec:experiments}

In this section, we present our setup and introduce our comparison metrics. We present experiments across four tabular datasets and compare the performance of the baselines with our method using evaluation scores and runtime efficiency. We also perform an experiment showcasing how SOGAR can be used to audit classifiers for bias. Finally, we present our observations and analyze the Pareto fronts generated, discussing the generalization and usefulness of Pareto set solutions to auditors.

\begin{wrapfigure}{r}{0.5\textwidth}
    \centering
    \includegraphics[width=\linewidth]{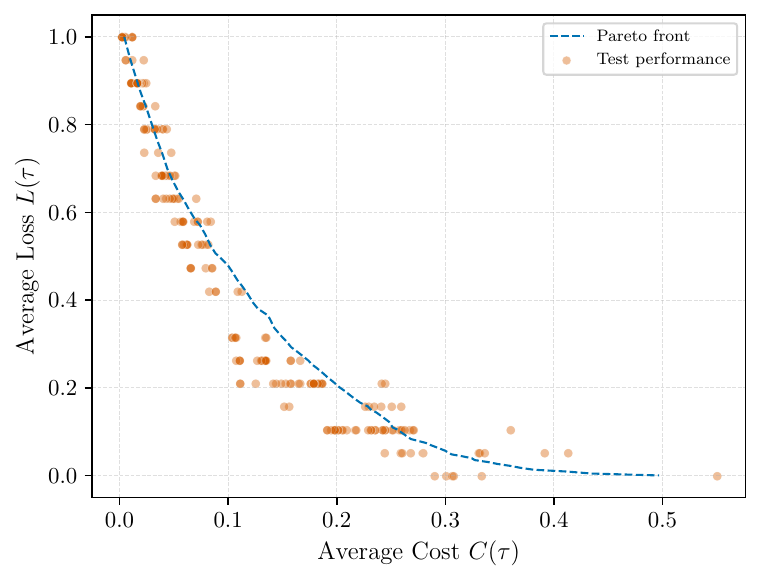}
    \caption{Pareto front on Employee Attrition (LightGBM, depth~3).
             Blue dashed line: train Pareto front. Orange dots:
             held-out test metrics. Average Euclidean distance: $0.04 \pm 0.02$ .}
    \label{fig:pareto_exp}
\end{wrapfigure}

\subsection{Experimental Setup}

Experiments presented in this section were conducted on a machine with i9-14900KS CPU (24 cores), 128GB RAM, and RTX 4090 GPU (24GB VRAM). The SOGAR optimization task is implemented as a STreeD task in C++17. The remainder of the implementation, including pre-processing, binarization, and cache computation, was all implemented in Python 3.13. Total compute time was 150\,h in main experiments (13.6\,h SOGAR, 136.5\,h rest baselines ), and an additional 67.1\,h in ablation studies in Appendices \ref{app:timeouts}--\ref{app:ablations}. 

\paragraph{Datasets.} We examine four tabular datasets, appropriate for auditing the classifiers. These are the Employee Attrition\footnote{Database Contents License (DbCL) v1.0} \cite{Kaggle_IBM_HR_Analytics_2017}, the German Credit\footnote{\label{foot:cc-by-4} CC BY 4.0} \cite{statlog_german_credit_data_144}, the Bank Marketing\footnotemark[\value{footnote}] \cite{bank_marketing_222}, and the Adult Income\footnotemark[\value{footnote}] \cite{adult_2} datasets. The Attrition and German Credit datasets are small to medium in size, with 1,400 and 1,000 instances, respectively, while the Bank Marketing\footnote{`additional' variant obtained by https://archive.ics.uci.edu/dataset/222/bank+marketing} and the Adult datasets are large in size with 4,500 and 50,000 instances, respectively. For our method and the ones viable to control the actionability of features, we constrain sensitive features as immutable, such as Gender, Age, or Race of individuals.

\paragraph{Classifiers.} The predictive models we used for auditing across all our experiments are two state-of-the-art tree-ensemble classifiers, LightGBM \cite{ke2017lightgbm} and XGBoost \cite{xgboost}, and a simple Deep Neural Network classifier.
The parameters were set to default values for every model and dataset, as the objective was model auditing, not classification accuracy.
We further discuss the models chosen and the parameters set for each architecture in Appendix \ref{app:experiments}.

\paragraph{Baselines.}
We compare to 5 prior baselines, these of CET \cite{cet_Kanamori_Takagi_Kobayashi_Ike}, AReS \cite{rawal2020beyond} , GLOBE-CE \cite{Ley_Mishra_Magazzeni_2023}, GLANCE \cite{kavouras2025glanceglobalactionsnutshell} and T-CREx \cite{bewley2024counterfactual}. For each baseline, we run the experiments using the default optimal parameters set by the authors, as indicated in each experiment.

\paragraph{Evaluation Metrics.}
To demonstrate the direct comparison with baseline methods, we report the recourse cost based on the MPS function defined in Equation \eqref{eq:mps}, and evaluate the effectiveness of recourse based on the recourse loss defined in Equation \eqref{eq:single_action_loss}. Finally, we adopt the sum of the two metrics as the invalidity score used by \cite{cet_Kanamori_Takagi_Kobayashi_Ike}\footnote{The invalidity score defined by \cite{cet_Kanamori_Takagi_Kobayashi_Ike} is a scaled sum of cost and loss using a \(\gamma\) coefficient on the loss function. In their experiments, they explicitly set \(\gamma =1\), thus, we do the same in the reported metrics.}.
The invalidity metric, which equally weighs cost and loss, serves as a tie-breaker for solutions which do not Pareto-dominate one another.

\subsection{Quantitative Comparison}
\label{subsec:quantitative_results}

Table \ref{tab:results_summary} presents the comparison of SOGAR to the other baseline methods on the 4 aforementioned datasets. 
The metrics we report are the recourse cost, recourse loss, invalidity, and computational time.
We perform 10-fold cross validation on each dataset and report the averaged metrics.
CET, AReS, and SOGAR are computationally heavy methods, so a time limit of 1 hour was imposed, which CET and AReS exceeded in some datasets.
SOGAR never reaches this time limit
and recovers all Pareto-optimal solutions.

SOGAR always outperforms all other methods on invalidity.
It also has the lowest cost in every dataset besides Attrition.
The methods that most often outperform SOGAR in terms of loss are GLOBE-CE and GLANCE, but even in these cases SOGAR's loss is often a close 2nd or 3rd.

SOGAR's running time is higher than GLOBE-CE and GLANCE, but it produces the entire Pareto front in a single run (e.g., $\approx$6,000 solutions on Adult at $\approx$0.2s/solution), eliminating the need for repeated runs under different hyperparameter settings. Additionally, timeouts can be imposed with minimal quality degradation, as shown in Appendix~\ref{app:timeouts}. Ablation studies on tree depth, action sparsity, binning granularity, and timeout sensitivity are provided in Appendices~\ref{app:timeouts}--\ref{app:ablations}, and standard deviations across folds are reported in Table~\ref{tab:results_summary_std} of 
Appendix~\ref{app:experiments}. 

\newcommand{\bad}[1]{\textcolor{red}{#1}}
\newcommand{\best}[1]{\textbf{#1}}

\begin{table}[t!]
\centering
\caption{Evaluating SOGAR against competing methods on 4 datasets and 3 classifiers with metrics averaged over 10 folds. Metrics reported are recourse cost (MPS), recourse loss, invalidity (their sum), and computational time in seconds. Lower is better for all metrics. Out of all solutions produced by SOGAR, metrics are reported for the one with the lowest invalidity. A dash indicates that the algorithm runtime exceeded the 1 hour timeout, and no metrics were collected.}
\label{tab:results_summary}
\large
\setlength{\tabcolsep}{2pt}
\renewcommand{\arraystretch}{0.4}

\resizebox{\textwidth}{!}{%
\begin{tabular}{ll cccc cccc cccc cccc}
\toprule
\textbf{Models} & \textbf{Algorithms} &
\multicolumn{4}{c}{\textbf{Attrition}} &
\multicolumn{4}{c}{\textbf{German}} &
\multicolumn{4}{c}{\textbf{Bank}} &
\multicolumn{4}{c}{\textbf{Adult}} \\
\cmidrule(lr){3-6}\cmidrule(lr){7-10}\cmidrule(lr){11-14}\cmidrule(lr){15-18}
 &  &
\textit{cost} $\downarrow$ & \textit{loss} $\downarrow$ & \textit{inv.} $\downarrow$ & \textit{time(s)} $\downarrow$ &
\textit{cost} $\downarrow$ & \textit{loss} $\downarrow$ & \textit{inv.} $\downarrow$ & \textit{time(s)} $\downarrow$ &
\textit{cost} $\downarrow$ & \textit{loss} $\downarrow$ & \textit{inv.} $\downarrow$ & \textit{time(s)} $\downarrow$ &
\textit{cost} $\downarrow$ & \textit{loss} $\downarrow$ & \textit{inv.} $\downarrow$ & \textit{time(s)} $\downarrow$ \\
\midrule

\multirow{7}{*}{DNN}
& CET        & 0.30 & 0.10 & 0.40 & 198.7 & 0.04 & 0.44 & 0.48 & 255.3 & 0.27 & 0.60 & 0.86 & 378.9 & -- & -- & -- & -- \\
& AReS       & 0.23 & 0.23 & 0.46 & 139.3 & 0.17 & 0.39 & 0.56 & 1,845.3 & -- & -- & -- & -- & -- & -- & -- & -- \\
& GLOBE-CE   & 0.81 & $<\!$0.01 & 0.81 & 7.6 & 0.91 & 0.01 & 0.92 & 7.58 & 0.87 & 0.04 & 0.89 & 23.2 & 0.98 & 0.00 & 0.98 & 170.5 \\
& GLANCE     & 0.54 & $<\!$0.01 & 0.54 & 11.66 & 0.47 & $<\!$0.01 & 0.47 & 8.5 & 0.65 & $<\!$0.01 & 0.66 & 10.75 & 0.96 & 0.00 & 0.96 & 47.22 \\
& T-CREx$_{0.9}$  & 0.35 & 0.77 & 1.13 & 7.0 & 0.53 & 0.78 & 1.31 & 32.19 & -- & -- & -- & -- & 0.66 & 0.91 & 1.58 & 15.04 \\
& T-CREx$_{0.99}$ & 0.56 & 0.65 & 1.21 & 176.32 & 0.60 & 0.70 & 1.25 & 19.63 & -- & -- & -- & -- & 0.64 & 0.86 & 1.50 & 10.4 \\
& \best{SOGAR} & \textbf{0.05} & \textbf{$<\!$0.01} & \textbf{0.05} & 250.0 & \textbf{0.02} & \textbf{0.00} & \textbf{0.02} & 28.2 & \textbf{0.17} & 0.29 & \textbf{0.46} & 268.9 & \textbf{0.03} & 0.03 & \textbf{0.06} & 1,248.4 \\
\midrule

\multirow{7}{*}{XGBoost}
& CET        & 0.40 & 0.17 & 0.57 & 342.6 & 0.12 & 0.25 & 0.37 & 151.2 & 0.51 & 0.24 & 0.75 & 438.8 & -- & -- & -- & -- \\
& AReS       & \textbf{0.02} & 0.97 & 0.99 & 11.68 & 0.41 & 0.13 & 0.55 & 2,362.8 & -- & -- & -- & -- & -- & -- & -- & -- \\
& GLOBE-CE   & 0.75 & \textbf{0.00} & 0.75 & \textbf{5.90} & 0.92 & 0.25 & 1.16 & \textbf{5.5} & 0.91 & \textbf{0.01} & 0.92 & 30.91 & 0.98 & \textbf{0.00} & 0.98 & 205.2 \\
& GLANCE     & 0.56 & 0.02 & 0.58 & 11.70 & 0.50 & \textbf{$<\!$0.01} & 0.50 & 8.4 & 0.69 & 0.03 & 0.72 & 10.6 & 0.96 & $<\!$0.01 & 0.96 & 35.1 \\
& T-CREx$_{0.9}$  & 0.33 & 0.72 & 1.10 & 7.99 & 0.40 & 0.75 & 1.20 & 8.2 & 0.32 & 0.99 & 1.31 & 6.87 & 0.70 & 0.90 & 1.46 & 12.52 \\
& T-CREx$_{0.99}$ & 0.48 & 0.64 & 1.12 & 14.13 & 0.56 & 0.57 & 1.13 & 8.03 & -- & -- & -- & -- & 0.35 & 0.99 & 1.13 & 10.23 \\
& \best{SOGAR} & 0.23 & 0.02 & \textbf{0.26} & 254.1 & \textbf{0.09} & 0.02 & \textbf{0.12} & 26.4 & \textbf{0.21} & 0.26 & \textbf{0.47} & 250.7 & \textbf{0.30} & 0.34 & \textbf{0.64} & 763.0 \\
\midrule

\multirow{7}{*}{LightGBM}
& CET        & \textbf{0.26} & 0.58 & 0.84 & 286.7 & 0.13 & 0.31 & 0.44 & 173.8 & 0.60 & 0.01 & 0.61 & 383.9 & -- & -- & -- & -- \\
& AReS       & 0.44 & 0.44 & 0.88 & 307.4 & 0.19 & 0.65 & 0.84 & 680.4 & 0.36 & 0.80 & 1.15 & 213.7 & -- & -- & -- & -- \\
& GLOBE-CE   & 0.79 & \textbf{0.01} & 0.80 & \textbf{7.96} & 0.91 & 0.23 & 1.14 & \textbf{6.1} & 0.92 & \textbf{$<\!$0.01} & 0.92 & 29.8 & 0.97 & \textbf{0.00} & 0.97 & 171.6 \\
& GLANCE     & 0.62 & 0.08 & 0.70 & 12.17 & 0.52 & \textbf{$<\!$0.01} & 0.53 & 9.3 & 0.71 & \textbf{$<\!$0.01} & 0.71 & \textbf{11.0} & 0.93 & $<\!$0.01 & 0.93 & \textbf{32.0} \\
& T-CREx$_{0.9}$  & 0.37 & 0.80 & 1.16 & 7.08 & 0.44 & 0.72 & 1.16 & 8.25 & 0.38 & 0.66 & 1.04 & 6.83 & 0.28 & 0.96 & 1.24 & 10.88 \\
& T-CREx$_{0.99}$ & 0.58 & 0.84 & 1.42 & 112.67 & 0.54 & 0.61 & 1.15 & 10.0 & 0.47 & 0.58 & 1.04 & 8.71 & 0.45 & 0.99 & 1.44 & 10.5 \\
& \best{SOGAR} & 0.30 & 0.07 & \textbf{0.37} & 356.4 & \textbf{0.11} & 0.03 & \textbf{0.14} & 41.9 & \textbf{0.19} & 0.06 & \textbf{0.25} & 166.2 & 0.35 & 0.17 & \textbf{0.52} & 1,235.3 \\
\bottomrule
\end{tabular}%
}
\end{table}

\subsection{Pareto Front and Generalization}

In Figure \ref{fig:pareto_exp}, the blue dashed line depicts the entire Pareto front of recourse summary trees of maximum depth 3, as computed by SOGAR on the training dataset.
The trade-off between the two metrics becomes explicitly quantified and an end-user can decide which point on the Pareto front best fits their needs.
We also plot metrics on a held-out portion of the same dataset to examine whether optimizing the objectives directly leads to over-fitting.
Each orange dot represents a SOGAR solution evaluated on this held-out dataset.
We can see that the orange dots are located very close to or on the dashed blue line, so there is little to no overfitting, with an average Euclidean distance of 0.04.

\subsection{Bias Auditing Experiment}
\label{subsec:bias_auditing_experiment_main}

\begin{figure}[t!]
    \centering
    \captionsetup[subfigure]{skip=2pt}
    \captionsetup{skip=3pt}
    \begin{subfigure}[t]{0.48\linewidth}
        \centering
        \includegraphics[width=\linewidth]{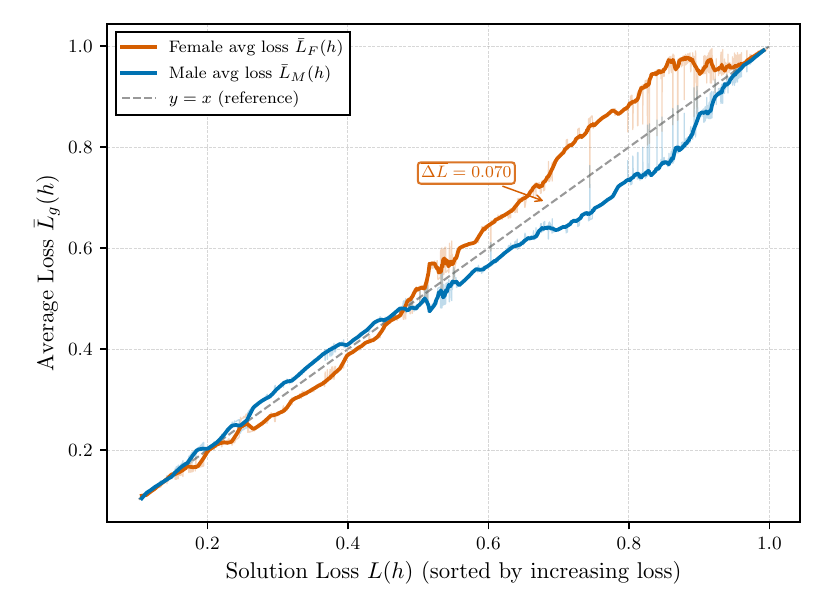}
        \caption{Loss ($\Delta L = 0.070$).}
        \label{fig:inst_loss_immutable}
    \end{subfigure}
    \hfill
    \begin{subfigure}[t]{0.48\linewidth}
        \centering
        \includegraphics[width=\linewidth]{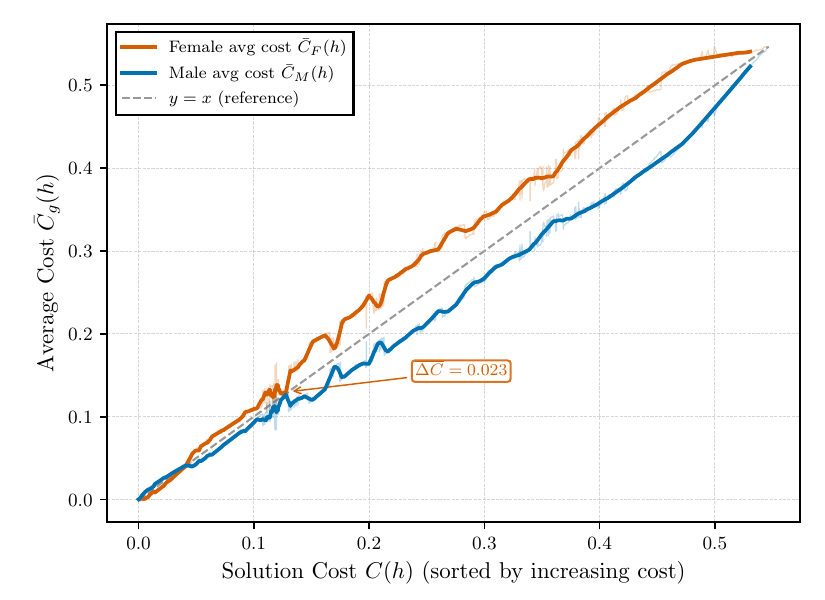}
        \caption{Cost ($\Delta C = 0.023$).}
        \label{fig:inst_cost_immutable}
    \end{subfigure}
    \caption{Loss and cost disparity across the Pareto front. The cost gap is persistent while the loss gap is more prominent.}
    \label{fig:loss_cost_disparity}
\end{figure}

\paragraph*{Setup} We run SOGAR on the Adult Income dataset~\cite{adult_2} (32,561 training instances) using the same LightGBM classifier and hyperparameters as the main experiments in Table~\ref{tab:results_summary} (depth~3, MaxNodes~7, MinLeaf~500). 
Grouping individuals by Sex, the classifier exhibits a Disparate Impact Ratio of $\mathbf{0.64}$, well below the $4/5$ 
threshold~\cite{watkins2024four} signifying a level of bias: 86.5\% of females receive adverse predictions versus 55.1\% of males.
Of the 19,199 adversely-predicted training instances, 43.8\% are female.
To further investigate the presence of bias we examine the differences between the recourse that females and males across the Pareto front of optimal recourse summary trees.
Per-group metrics are computed by routing every instance through all the Pareto front, assigning each instance the cost and loss of its landed leaf, then averaging per group and across the full front.

\paragraph{Findings.}
\begin{wrapfigure}{r}{0.5\linewidth}
    \vspace{-0.5em}
    \centering
    \includegraphics[width=\linewidth]{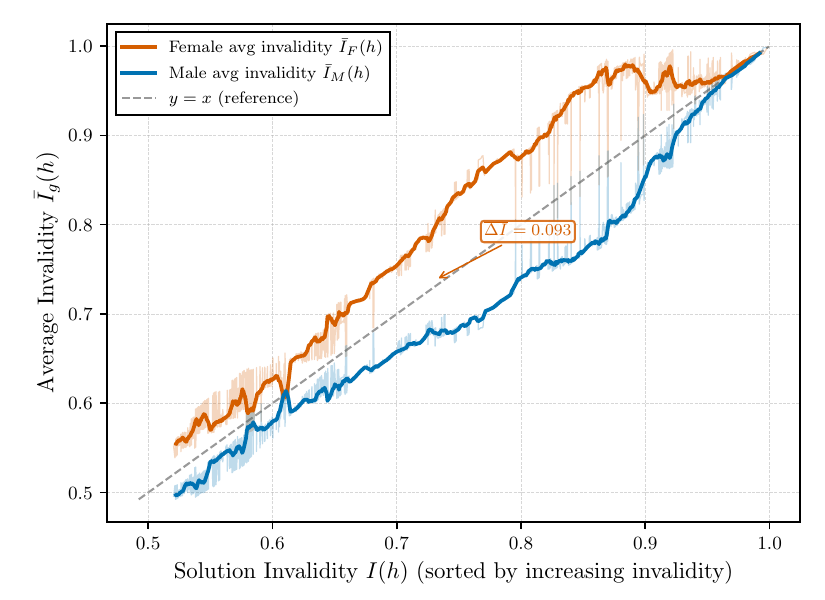}
    \caption{Invalidity ($\Delta I = 0.093$) across 
    the Pareto front. Female group average is 
    consistently above male across the entire front.}
    \label{fig:inst_inv_immutable}
    \vspace{-1em}
\end{wrapfigure}

SOGAR generates 4,142 Pareto-optimal solutions.
Across the front, \textbf{96.4\%} of solutions exhibit higher invalidity for females than males on training instances (mean gap $\Delta I = +0.093$; test instances: 95.9\%, $+0.081$). 
Figure~\ref{fig:inst_inv_immutable} shows that this gap persists across the entire invalidity range.
The decomposition in Figure~\ref{fig:loss_cost_disparity} reveals that cost is worse for females for almost of regions of the Pareto front, while loss is favorable towards females for low-loss regions.
Nevertheless, the loss disparities are wide enough in the high regions to contribute more to the total invalidity gap; $\Delta L = +0.070$, $\Delta C = +0.023$.
The invalidity disparity is widest above the mid-range and narrows toward both extremes of the front. Single-output methods such as GLOBE-CE~\cite{Ley_Mishra_Magazzeni_2023} and GLANCE~\cite{kavouras2025glanceglobalactionsnutshell} produce solutions in the low-loss and high-cost region where this gap is at its narrowest and would therefore understate the disparity that the full Pareto front exposes. Extended statistics are reported in Appendix~\ref{app:bias_audit_appendix}.

\section{Conclusion}
\label{sec:conclusion}

We introduced SOGAR, a method for computing recourse summaries by formulating the problem as bi-objective optimal decision tree learning.
Leveraging the STreeD dynamic programming framework, SOGAR produces the complete Pareto front of recourse summary trees, providing global optimality guarantees over both recourse cost and effectiveness.
Our experiments across four benchmark datasets demonstrate that SOGAR consistently outperforms existing recourse summary and global counterfactual explanation methods.
On the invalidity metric, which equally weights cost and loss, SOGAR achieves the best performance in all datasets.
The shallow, axis-parallel tree structure produces interpretable subgroup definitions, while the sparse actions assigned to each leaf remain actionable for affected individuals.
The diversity of the Pareto front solutions allows us to examine how recourse differs per group for different values of total cost or loss, something that single-solution methods might miss.

\section{Limitations}
\label{subsec:limitations}
We acknowledge several limitations of our approach.
First, SOGAR incurs higher computational costs than methods such as GLOBE-CE and GLANCE, though this is partially offset by producing an entire frontier of solutions rather than a single one.
For time-sensitive applications, the anytime property of our algorithm allows early termination with the best solutions found thus far, but the optimality guarantees are then lost.
Second, scalability to very large and high-dimensional datasets remains challenging.
The possible combinations of action-instance pairs get excessively large and the cache our method computes might exceed available hardware resources, and while STreeD's caching and pruning mitigate the combinatorial explosion of the search space, these datasets may require additional constraints or feature selection.
Finally, we point out that even though SOGAR can aid in auditing classifiers, it should not serve as the only criterion. Bias and fairness are complex multi-dimensional topics and multiple methods should be employed for a complete audit. Relying on just the results of SOGAR could potentially obscure bias hidden in complex relations between features or conflate representation bias within datasets with classifier bias.

\bibliographystyle{plain}

\bibliography{references}

\begin{thebibliography}{10}

\bibitem{aghaei2021strong}
Sina Aghaei, Andr{\'e}s G{\'o}mez, and Phebe Vayanos.
\newblock Strong optimal classification trees.
\newblock {\em arXiv preprint arXiv:2103.15965}, 2021.

\bibitem{aglin2020learning}
Ga{\"e}l Aglin, Siegfried Nijssen, and Pierre Schaus.
\newblock Learning optimal decision trees using caching branch-and-bound search.
\newblock In {\em Proceedings of the AAAI conference on artificial intelligence}, volume~34, pages 3146--3153, 2020.

\bibitem{adult_2}
Barry Becker and Ronny Kohavi.
\newblock Adult.
\newblock UCI Machine Learning Repository, 1996.
\newblock {DOI}: https://doi.org/10.24432/C5XW20.

\bibitem{bertsimas2017optimal}
Dimitris Bertsimas and Jack Dunn.
\newblock Optimal classification trees.
\newblock {\em Machine Learning}, 106(7):1039–1082, July 2017.

\bibitem{bertsimas2018interpretable}
Dimitris Bertsimas, Agni Orfanoudaki, and Holly Wiberg.
\newblock Interpretable clustering via optimal trees.
\newblock {\em arXiv preprint arXiv:1812.00539}, 2018.

\bibitem{bewley2024counterfactual}
Tom Bewley, Salim~I Amoukou, Saumitra Mishra, Daniele Magazzeni, and Manuela Veloso.
\newblock Counterfactual metarules for local and global recourse.
\newblock {\em arXiv preprint arXiv:2405.18875}, 2024.

\bibitem{breiman1984classification}
L~Breiman, JH~Friedman, R~Olshen, and CJ~Stone.
\newblock Classification and regression trees.
\newblock 1984.

\bibitem{xgboost}
Tianqi Chen and Carlos Guestrin.
\newblock Xgboost: A scalable tree boosting system.
\newblock In {\em Proceedings of the 22nd ACM SIGKDD International Conference on Knowledge Discovery and Data Mining}, KDD '16, page 785–794, New York, NY, USA, 2016. Association for Computing Machinery.

\bibitem{Chou_Moreira_Bruza_Ouyang_Jorge_2022}
Yu-Liang Chou, Catarina Moreira, Peter Bruza, Chun Ouyang, and Joaquim Jorge.
\newblock Counterfactuals and causability in explainable artificial intelligence: Theory, algorithms, and applications.
\newblock {\em Information Fusion}, 81:59–83, May 2022.

\bibitem{demirovic2022murtree}
Emir Demirovi{\'c}, Anna Lukina, Emmanuel Hebrard, Jeffrey Chan, James Bailey, Christopher Leckie, Kotagiri Ramamohanarao, and Peter~J Stuckey.
\newblock Murtree: Optimal decision trees via dynamic programming and search.
\newblock {\em Journal of Machine Learning Research}, 23(26):1--47, 2022.

\bibitem{guidotti2018survey}
Riccardo Guidotti, Anna Monreale, Salvatore Ruggieri, Franco Turini, Fosca Giannotti, and Dino Pedreschi.
\newblock A survey of methods for explaining black box models.
\newblock {\em ACM computing surveys (CSUR)}, 51(5):1--42, 2018.

\bibitem{statlog_german_credit_data_144}
Hans Hofmann.
\newblock Statlog - german credit data.
\newblock UCI Machine Learning Repository, 1994.
\newblock {DOI}: https://doi.org/10.24432/C5NC77.

\bibitem{hu2019optimal}
Xiyang Hu, Cynthia Rudin, and Margo Seltzer.
\newblock Optimal sparse decision trees.
\newblock {\em Advances in neural information processing systems}, 32, 2019.

\bibitem{hyafil1976optimal}
Laurent Hyafil and Ronald~L. Rivest.
\newblock Constructing optimal binary decision trees is np-complete.
\newblock {\em Information Processing Letters}, 5(1):15--17, 1976.

\bibitem{cet_Kanamori_Takagi_Kobayashi_Ike}
Kentaro Kanamori, Takuya Takagi, Ken Kobayashi, and Yuichi Ike.
\newblock Counterfactual explanation trees: Transparent and consistent actionable recourse with decision trees.
\newblock In {\em International Conference on Artificial Intelligence and Statistics}, pages 1846--1870. PMLR, 2022.

\bibitem{Karimi_Scholkopf_Valera_2021}
Amir-Hossein Karimi, Bernhard Schölkopf, and Isabel Valera.
\newblock Algorithmic recourse: from counterfactual explanations to interventions.
\newblock In {\em Proceedings of the 2021 ACM Conference on Fairness, Accountability, and Transparency}, FAccT ’21, page 353–362, New York, NY, USA, March 2021. Association for Computing Machinery.

\bibitem{kavouras2025glanceglobalactionsnutshell}
Loukas Kavouras, Eleni Psaroudaki, Konstantinos Tsopelas, Dimitrios Rontogiannis, Nikolaos Theologitis, Dimitris Sacharidis, Giorgos Giannopoulos, Dimitrios Tomaras, Kleopatra Markou, Dimitrios Gunopulos, Dimitris Fotakis, and Ioannis Emiris.
\newblock Glance: Global actions in a nutshell for counterfactual explainability, 2025.

\bibitem{ke2017lightgbm}
Guolin Ke, Qi~Meng, Thomas Finley, Taifeng Wang, Wei Chen, Weidong Ma, Qiwei Ye, and Tie-Yan Liu.
\newblock Lightgbm: A highly efficient gradient boosting decision tree.
\newblock {\em Advances in neural information processing systems}, 30, 2017.

\bibitem{Ley_Mishra_Magazzeni_2023}
Dan Ley, Saumitra Mishra, and Daniele Magazzeni.
\newblock Globe-ce: A translation-based approach for global counterfactual explanations.
\newblock (arXiv:2305.17021), December 2023.
\newblock arXiv:2305.17021 [cs].

\bibitem{Lin_Zhong_Hu_Rudin_Seltzer_2022}
Jimmy Lin, Chudi Zhong, Diane Hu, Cynthia Rudin, and Margo Seltzer.
\newblock Generalized and scalable optimal sparse decision trees.
\newblock (arXiv:2006.08690), November 2022.
\newblock arXiv:2006.08690 [cs].

\bibitem{Lundberg_Lee_2017}
Scott Lundberg and Su-In Lee.
\newblock A unified approach to interpreting model predictions.
\newblock (arXiv:1705.07874), November 2017.
\newblock arXiv:1705.07874 [cs].

\bibitem{mctavish2022fast}
Hayden McTavish, Chudi Zhong, Reto Achermann, Ilias Karimalis, Jacques Chen, Cynthia Rudin, and Margo Seltzer.
\newblock Fast sparse decision tree optimization via reference ensembles.
\newblock In {\em Proceedings of the AAAI conference on artificial intelligence}, volume~36, pages 9604--9613, 2022.

\bibitem{bank_marketing_222}
S.~Moro, P.~Rita, and P.~Cortez.
\newblock Bank marketing.
\newblock UCI Machine Learning Repository, 2014.
\newblock {DOI}: https://doi.org/10.24432/C5K306.

\bibitem{narodytska2018learning}
Nina Narodytska, Alexey Ignatiev, Filipe Pereira, and Joao Marques-Silva.
\newblock Learning optimal decision trees with sat.
\newblock In {\em International Joint Conference on Artificial Intelligence 2018}, pages 1362--1368. Association for the Advancement of Artificial Intelligence (AAAI), 2018.

\bibitem{nijssen2007mining}
Siegfried Nijssen and Elisa Fromont.
\newblock Mining optimal decision trees from itemset lattices.
\newblock In {\em Proceedings of the 13th ACM SIGKDD international conference on Knowledge discovery and data mining}, pages 530--539, 2007.

\bibitem{quinlan1993c4}
J~Ross Quinlan.
\newblock {\em C4. 5: Programs for Machine Learning}.
\newblock Morgan Kaufmann, 1993.

\bibitem{rawal2020beyond}
Kaivalya Rawal and Himabindu Lakkaraju.
\newblock Beyond individualized recourse: Interpretable and interactive summaries of actionable recourses.
\newblock {\em Advances in Neural Information Processing Systems}, 33:12187--12198, 2020.

\bibitem{ribeiro2016should}
Marco~Tulio Ribeiro, Sameer Singh, and Carlos Guestrin.
\newblock " why should i trust you?" explaining the predictions of any classifier.
\newblock In {\em Proceedings of the 22nd ACM SIGKDD international conference on knowledge discovery and data mining}, pages 1135--1144, 2016.

\bibitem{rudin2019stop}
Cynthia Rudin.
\newblock Stop explaining black box machine learning models for high stakes decisions and use interpretable models instead.
\newblock {\em Nature machine intelligence}, 1(5):206--215, 2019.

\bibitem{Kaggle_IBM_HR_Analytics_2017}
Pavan Subhash.
\newblock Ibm hr analytics employee attrition \& performance dataset.
\newblock \url{https://www.kaggle.com/datasets/pavansubhasht/ibm-hr-analytics-attrition-dataset}, 2017.
\newblock Accessed: 2025-11-04.

\bibitem{Ustun_Spangher_Liu_2019}
Berk Ustun, Alexander Spangher, and Yang Liu.
\newblock Actionable recourse in linear classification.
\newblock In {\em Proceedings of the Conference on Fairness, Accountability, and Transparency}, FAT* ’19, page 10–19, New York, NY, USA, January 2019. Association for Computing Machinery.

\bibitem{streed_VAN_2023}
Jacobus van~der Linden, Mathijs de~Weerdt, and Emir Demirovi{\'c}.
\newblock Necessary and sufficient conditions for optimal decision trees using dynamic programming.
\newblock {\em Advances in Neural Information Processing Systems}, 36:9173--9212, 2023.

\bibitem{wachter2018counterfactualexplanationsopeningblack}
Sandra Wachter, Brent Mittelstadt, and Chris Russell.
\newblock Counterfactual explanations without opening the black box: Automated decisions and the gdpr, 2018.

\bibitem{watkins2024four}
Elizabeth~Anne Watkins and Jiahao Chen.
\newblock The four-fifths rule is not disparate impact: a woeful tale of epistemic trespassing in algorithmic fairness.
\newblock In {\em Proceedings of the 2024 ACM Conference on Fairness, Accountability, and Transparency}, pages 764--775, 2024.

\bibitem{zhang2024optimal}
Rui Zhang, Rui Xin, Margo Seltzer, and Cynthia Rudin.
\newblock Optimal sparse survival trees.
\newblock {\em Proceedings of machine learning research}, 238:352, 2024.

\end{thebibliography}

\appendix

\section{Proof of Proposition \ref{prop:optimality}} \label{app:proof}

We prove Proposition \ref{prop:optimality} by showing that recourse summary tree optimization is a \emph{separable} optimization task in the framework of \cite{streed_VAN_2023}. By their Theorem 4.7 this guarantees that dynamic programming computes the exact Pareto front.

\subsection{Defining the Optimization Task}

We define the optimization task $o = \langle g, t, \succ, \oplus, c, s_0 \rangle$ as follows.

\paragraph{State Space} A state is a pair $s = \langle \mathcal{D}_s, \mathcal{F}_s \rangle$.
$\mathcal{D}_s \subseteq \mathcal{D}_B$ is a subset of the (original, binarized) instance pairs $\mathcal{D}_s \subseteq \mathcal{D}_B$.
$\mathcal{F}_s \subseteq \mathcal{F}$ is the set of binary features used in branching decisions of ancestor nodes.

\paragraph{Solution Space} A solution is a pair $v = (v_C, v_L) \in \mathbb{R}_{\geq 0} \times \mathbb{Z}$.
$v_C$ is the action cost aggregated over the instances reaching that leaf and $v_L$ is the recourse loss aggregated in the same way.

\paragraph{Initial State} $s_0 = \langle \mathcal{D}_B, \varnothing \rangle$

\paragraph{Cost Function} For a state $s = \langle \mathcal{D}_s, \mathcal{F}_s \rangle$, the cost function $g$ is 0 when branching on features:
\begin{equation}
    g(s, f) = (0, 0), \quad f \in \mathcal{F} \label{eq:cost1}
\end{equation}
It is equal to the aggregated costs when assigning actions to leaves:
\begin{equation}
    g(s, a) = \left( \sum_{(\phi(x),  x) \in \mathcal{D}_s}  c(a, x), \sum_{(\phi(x),  x) \in \mathcal{D}_s}  l(a, x) \right), \quad a \in \mathcal{A} \label{eq:cost2}
\end{equation}

\paragraph{Transition Function} For a state $s = \langle \mathcal{D}_s, \mathcal{F}_s \rangle$, when branching on feature $f \in \mathcal{F} \setminus \mathcal{F}_s$ the transition function $t$ is:
\begin{align}
t(s, f) &= \langle \{(\phi(x), x) \in \mathcal{D}_s : \phi(x)_f = 0\}, \mathcal{F}_s \cup \{f\} \rangle \label{eq:transition1} \\
t(s, \bar{f}) &= \langle \{(\phi(x), x) \in \mathcal{D}_s : \phi(x)_f = 1\}, \mathcal{F}_s \cup \{f\} \rangle \label{eq:transition2}
\end{align}

\paragraph{Comparison Operator (Pareto Dominance)} The comparison operator $\succ$ is the strict Pareto dominance for element-wise comparison:
\begin{equation}
    (v_C, v_L) \succ (v_C', v_l') \Leftrightarrow (v_C \leq v_C') \wedge (v_L \leq v_L') \wedge ((v_C, v_L) \neq (v_C', v_l')) \label{eq:comparison}
\end{equation}

In other words, at least one of the inequalities is not strict. Note that $\succ$ ``translates'' to $\leq$ (flipped order), because $v \succ v'$ means ``$v$ is preferred to $v'$'' and therefore at least one of the minimization objectives is lower.

\paragraph{Combining Operator} The combining operator $\oplus$ is element-wise addition:
\begin{equation}
    (v_C, v_L) \oplus (v_C', v_L') = (v_C + v_C', v_L + v_L') \label{eq:addition}
\end{equation}

\paragraph{Constraint} We consider the unconstrained case, $c(v, s) =1, ~\forall v, s$.

In our implementation, we also use maximum leaf count and minimum leaf size as constraints, but these can be incorporated without affecting separability, as demonstrated in \cite{streed_VAN_2023}. Maximum leaf count is especially useful, since it imposes a bound on the total number of population subgroups which improves the interpretability of the recourse summary tree.

\subsection{Verification of Separability Conditions}

By Theorem 4.6 of \cite{streed_VAN_2023}, an optimization task is separable if and only if its cost function $g$ and transition function $t$ are Markovian, its combining operator $\oplus$ preserves order over its comparison operator $\succ$ and the constraint $c$ is anti-monotonic.

The cost function $g$ and transition function $t$, as defined in equations \eqref{eq:cost1}, \eqref{eq:cost2}, \eqref{eq:transition1} and \eqref{eq:transition2}, are indeed Markovian as by their definition they only depend on the current state and the selected action/feature. The trivial constraint $c(v,s) = 1$ is also trivially anti-monotonic.

\begin{lemma}[Order Preserving Combining Operator]
    The combining operator $\oplus$, as defined in Equation \eqref{eq:addition}, preserves the order of the comparison operator $\succ$ as defined in Equation \eqref{eq:comparison}.
\end{lemma}

\begin{proof}
    Let $v = (v_C, v_L)$, $v' = (v_C', v_L')$, and $u = (u_C, u_L)$ be solution values with $v \succ v'$.
    We must show that $v \oplus u \succ v' \oplus u$. By the definitions of the operators and the properties of element-wise addition, we have:

    \begin{itemize}
        \item $v_C \leq v_C'$ and $v_L \leq v_L'$, with at least one strict.
        \item $v_C + u_C \leq v_C' + u_C$ and $v_L + u_L \leq v_L' + u_L$, with at least one strict.
        \item Therefore $v \oplus u \succ v' \oplus u$
    \end{itemize}
    
    Since $\oplus$ is commutative, it also holds that $u \oplus v \succ u \oplus v'$.
\end{proof}

We have verified all necessary condition of Theorem 4.6, therefore recourse summary tree computation is separable. By Theorem 4.7 of \cite{streed_VAN_2023}, their dynamic programming formulation, termed \emph{STreeD}, computes the complete Pareto front of recourse summary trees:
\begin{equation}
T(s, d) = \begin{cases}
\text{opt}\left( \bigcup_{a \in \mathcal{A}} \{g(s, a)\}, s \right) & \text{if } d = 0\\
\text{opt}\left( \bigcup_{f \in \mathcal{F}} \text{merge}\left( T(t(s,f), d-1), T(t(s,\bar{f}), d-1), s, f \right), s \right) & \text{if } d > 0
\end{cases}
\end{equation}
Where the definitions of $\text{opt}$ and $\text{merge}$ are taken directly from \cite{streed_VAN_2023} and are the following:
\begin{align}
    \text{merge}(\Theta_1, \Theta_2, s) &= \{ v_1 \oplus v_2 \oplus g(s,f) ~|~ v_1 \in \Theta_1, v_2 \in \Theta_2 \}\\
    \text{opt}(\Theta, s) &= \text{nondom}(\text{feas}(\Theta, s))\\
    \text{nondom}(\Theta) &= \{ v\in \Theta ~|~ \nexists v' \in \Theta (u' \succ v) \}\\
    \text{feas}(\Theta, s) &= \{ v \in \Theta ~|~ c(v,s) = 1 \}
\end{align}

This concludes the proof of Proposition \ref{prop:optimality}. \qed

\section{Dataset Preprocessing, Model Parameters \& Analytical Quantitative Results} 
\label{app:experiments}

\subsection{Datasets.}
In our experiments we constrained each dataset to provide practical actions, suitable for real-world deployment, by prohibiting the actionability of sensitive features (e.g., Gender, Race, Age) and features that cannot be proposed to be altered, in specific settings for affected individuals, such as the Type or Level of Education, or the Distance of an employee's home from their work place in the Employee Attrition Setting. We also set directionality of actions, which respects features such as time-dependent features or financial status features, where an action of lowering one's income would be unfavorable. Below, we state the actionability of each dataset based on the percentage of features made immutable, the directionality constraints, and the number of bins allowed to be changed per feature. 

\begin{table}[ht]
\centering
\caption{Actionability configuration per dataset. Immutability \% is
computed after one-hot encoding of categorical features.}
\label{tab:actionability}
\small
\begin{tabular}{lcc}
\toprule
\textbf{Dataset} & \textbf{Immutable (\%)} &
\textbf{Actions ($>$)} \\
\midrule
Attrition & 26\% & 120{,}000 \\
German    & 20\% & 90{,}000  \\
Bank      & 65\% & 5{,}000   \\
Adult     & 50\% & 100{,}000 \\
\bottomrule
\end{tabular}
\end{table}

\subsection{Classifiers \& SOGAR parameters} 
\paragraph{Classifiers. }We used two state-of-the-art classifiers, LightGBM and XGBoost, and a DNN classifier, all of which employ a black-box architecture. The parameters used to train the models were standard across all datasets, and methods run in our experiments.
Specifically, the LightGBM Classifier, was set to \texttt{n\_estimators} \(=100\) and \texttt{n\_leaves}\(=16\), the XGBoost was set to \texttt{n\_estimators} \(=100\) and \texttt{max\_depth}\(=6\), and the DNN was set to have \texttt{depth} \(=5\), \texttt{width} $=50$ and \texttt{dropout}\(=0.1\). The rest of the parameters were set to default values.

\paragraph{SOGAR default settings. }SOGAR's key hyperparameters are the maximum tree depth $d$, maximum
number of leaf nodes $m$, minimum leaf size $\ell$, and action sparsity
$k$ (maximum number of feature changes per action). All main-paper
experiments use the following defaults depicted in Table \ref{tab:sogar_params}. The larger minimum leaf size for Adult (500 vs.\ 50) is necessary to
control the search space on a dataset with over 30,000 training instances and
161 binarised features. Ablations on depth, sparsity, and bins are
reported in Appendix~\ref{app:ablations}.

\begin{table}[h]
\centering
\caption{SOGAR default hyperparameters for all main experiments.}
\label{tab:sogar_params}
\footnotesize
\begin{tabular}{ll}
\toprule
\textbf{Parameter} & \textbf{Value} \\
\midrule
Max depth $d$           & 3 \\
Max nodes $m$           & 7 \\
Min leaf size $\ell$    & 50 (500 for Adult) \\
Action sparsity $k$     & 3 \\
Feature bins            & 10--50 (distribution-adaptive) \\
Timeout                 & None (full Pareto front) \\
\bottomrule
\end{tabular}
\end{table}

\subsection{Extended Quantitative Results of Table~\ref{tab:results_summary}}

In this part of the Appendix, we present the extended variant of Table~\ref{tab:results_summary} that presents the results of the 10-fold cross validation across the four datasets of Employee Attrition, German Credit, Bank Marketing, and Adult Income, and is used to compare SOGAR with baselines. For the integrity of our quantitative experiments, the extended variants include the standard deviation for all four metrics recorded in Table~\ref{tab:results_summary}, average \emph{cost, loss, invalidity,} and the \emph{computation time} with the standard deviation. 

\begin{table}[t!]
\centering
\caption{Evaluating SOGAR against competing methods on 4 datasets and 3
classifiers, with metrics averaged over 10 folds (mean$\pm$std). Lower is
better for all metrics. SOGAR reports the solution with lowest invalidity.
A dash indicates timeout exceeded.}
\label{tab:results_summary_std}
\scriptsize
\setlength{\tabcolsep}{4pt}
\renewcommand{\arraystretch}{0.92}

\begin{subtable}{\linewidth}
\centering
\caption{Attrition and German Credit}
\label{tab:results_summary_a}
\resizebox{\linewidth}{!}{%
\begin{tabular}{ll cccc cccc}
\toprule
\textbf{Models} & \textbf{Algorithms} &
\multicolumn{4}{c}{\textbf{Attrition}} &
\multicolumn{4}{c}{\textbf{German}} \\
\cmidrule(lr){3-6}\cmidrule(lr){7-10}
 & &
\textit{cost}$\downarrow$ & \textit{loss}$\downarrow$ & \textit{inv.}$\downarrow$ & \textit{time(s)}$\downarrow$ &
\textit{cost}$\downarrow$ & \textit{loss}$\downarrow$ & \textit{inv.}$\downarrow$ & \textit{time(s)}$\downarrow$ \\
\midrule
\multirow{7}{*}{DNN}
& CET             & 0.30$\pm$0.12 & 0.10$\pm$0.07 & 0.40$\pm$0.16 & 198.7$\pm$18.6 & 0.04$\pm$0.02 & 0.44$\pm$0.13 & 0.48$\pm$0.12 & 255.3$\pm$68.2 \\
& AReS            & 0.23$\pm$0.04 & 0.23$\pm$0.12 & 0.46$\pm$0.13 & 139.3$\pm$6.2  & 0.17$\pm$0.30 & 0.39$\pm$0.07 & 0.56$\pm$0.30 & $1{,}845.3\pm210$ \\
& GLOBE-CE        & 0.81$\pm$0.02 & $<\!$0.01$\pm$0.02 & 0.81$\pm$0.02 & 7.6$\pm$0.25 & 0.91$\pm$0.003 & 0.01$\pm$0.02 & 0.92$\pm$0.02 & 7.58$\pm$0.32 \\
& GLANCE          & 0.54$\pm$0.06 & $<\!$0.01$\pm$0.002 & 0.54$\pm$0.06 & 11.66$\pm$0.22 & 0.47$\pm$0.16 & $<\!$0.01$\pm$0.004 & 0.47$\pm$0.16 & 8.5$\pm$0.4 \\
& T-CREx$_{0.9}$  & 0.35$\pm$0.11 & 0.77$\pm$0.12 & 1.13$\pm$0.21 & 7.0$\pm$0.33   & 0.53$\pm$0.12 & 0.78$\pm$0.08 & 1.31$\pm$0.13 & 32.19$\pm$73.8 \\
& T-CREx$_{0.99}$ & 0.56$\pm$0.11 & 0.65$\pm$0.15 & 1.21$\pm$0.16 & 176.32$\pm$230 & 0.60$\pm$0.08 & 0.70$\pm$0.15 & 1.25$\pm$0.15 & 19.63$\pm$16 \\
& \best{SOGAR}    & \textbf{0.05$\pm$0.002} & \textbf{$<\!$0.01$\pm$0.001} & \textbf{0.05$\pm$0.003} & 250.0$\pm$6.2 & \textbf{0.02$\pm$0.004} & \textbf{0.00$\pm$0.00} & \textbf{0.02$\pm$0.003} & 28.2$\pm$0.9 \\
\midrule
\multirow{7}{*}{XGBoost}
& CET             & 0.40$\pm$0.10 & 0.17$\pm$0.14 & 0.57$\pm$0.11 & 342.6$\pm$29.3  & 0.12$\pm$0.03 & 0.25$\pm$0.05 & 0.37$\pm$0.05 & 151.2$\pm$35.7 \\
& AReS            & \textbf{0.02$\pm$0.02} & 0.97$\pm$0.04 & 0.99$\pm$0.04 & 11.68$\pm$2 & 0.41$\pm$0.06 & 0.13$\pm$0.05 & 0.55$\pm$0.06 & $2{,}362.8\pm123.3$ \\
& GLOBE-CE        & 0.75$\pm$0.03 & \textbf{0.00$\pm$0.00} & 0.75$\pm$0.03 & \textbf{5.90$\pm$0.21} & 0.92$\pm$0.003 & 0.25$\pm$0.17 & 1.16$\pm$0.17 & \textbf{5.5$\pm$0.0} \\
& GLANCE          & 0.56$\pm$0.12 & 0.02$\pm$0.03 & 0.58$\pm$0.11 & 11.70$\pm$0.11  & 0.50$\pm$0.15 & \textbf{$<\!$0.01$\pm$0.003} & 0.50$\pm$0.15 & 8.4$\pm$0.07 \\
& T-CREx$_{0.9}$  & 0.33$\pm$0.08 & 0.72$\pm$0.18 & 1.10$\pm$0.25 & 7.99$\pm$1.14  & 0.40$\pm$0.15 & 0.75$\pm$0.18 & 1.20$\pm$0.20 & 8.2$\pm$1.40 \\
& T-CREx$_{0.99}$ & 0.48$\pm$0.17 & 0.64$\pm$0.26 & 1.12$\pm$0.32 & 14.13$\pm$2.65 & 0.56$\pm$0.07 & 0.57$\pm$0.15 & 1.13$\pm$0.22 & 8.03$\pm$28 \\
& \best{SOGAR}    & 0.23$\pm$0.02 & 0.02$\pm$0.02 & \textbf{0.26$\pm$0.01} & 254.1$\pm$13.6 & \textbf{0.09$\pm$0.009} & 0.02$\pm$0.007 & \textbf{0.12$\pm$0.01} & 26.4$\pm$1.4 \\
\midrule
\multirow{7}{*}{LightGBM}
& CET             & \textbf{0.26$\pm$0.18} & 0.58$\pm$0.24 & 0.84$\pm$0.09 & 286.7$\pm$20.1   & 0.13$\pm$0.04 & 0.31$\pm$0.11 & 0.44$\pm$0.10 & 173.8$\pm$13.1 \\
& AReS            & 0.44$\pm$0.05 & 0.44$\pm$0.10 & 0.88$\pm$0.06 & 307.4$\pm$22.7   & 0.19$\pm$0.08 & 0.65$\pm$0.13 & 0.84$\pm$12   & 680.4$\pm$44.6 \\
& GLOBE-CE        & 0.79$\pm$0.02 & \textbf{0.01$\pm$0.008} & 0.80$\pm$0.03 & \textbf{7.96$\pm$0.28} & 0.91$\pm$0.0 & 0.23$\pm$0.30 & 1.14$\pm$0.30 & \textbf{6.1$\pm$0.2} \\
& GLANCE          & 0.62$\pm$0.02 & 0.08$\pm$0.04 & 0.70$\pm$0.04 & 12.17$\pm$0.5    & 0.52$\pm$0.14 & \textbf{$<\!$0.01$\pm$0.003} & 0.53$\pm$0.14 & 9.3$\pm$0.29 \\
& T-CREx$_{0.9}$  & 0.37$\pm$0.12 & 0.80$\pm$0.13 & 1.16$\pm$0.23 & 7.08$\pm$0.29    & 0.44$\pm$0.09 & 0.72$\pm$0.09 & 1.16$\pm$0.10 & 8.25$\pm$1.95 \\
& T-CREx$_{0.99}$ & 0.58$\pm$0.06 & 0.84$\pm$0.12 & 1.42$\pm$0.17 & 112.67$\pm$114.9 & 0.54$\pm$0.04 & 0.61$\pm$0.30 & 1.15$\pm$0.32 & 10.0$\pm$1.4 \\
& \best{SOGAR}    & 0.30$\pm$0.02 & 0.07$\pm$0.02 & \textbf{0.37$\pm$0.02} & 356.4$\pm$9.4 & \textbf{0.11$\pm$0.005} & 0.03$\pm$0.007 & \textbf{0.14$\pm$0.009} & 41.9$\pm$2.5 \\
\bottomrule
\end{tabular}}
\end{subtable}

\vspace{0.8em}

\begin{subtable}{\linewidth}
\centering
\caption{Bank Marketing and Adult Income}
\label{tab:results_summary_b}
\resizebox{\linewidth}{!}{%
\begin{tabular}{ll cccc cccc}
\toprule
\textbf{Models} & \textbf{Algorithms} &
\multicolumn{4}{c}{\textbf{Bank}} &
\multicolumn{4}{c}{\textbf{Adult}} \\
\cmidrule(lr){3-6}\cmidrule(lr){7-10}
 & &
\textit{cost}$\downarrow$ & \textit{loss}$\downarrow$ & \textit{inv.}$\downarrow$ & \textit{time(s)}$\downarrow$ &
\textit{cost}$\downarrow$ & \textit{loss}$\downarrow$ & \textit{inv.}$\downarrow$ & \textit{time(s)}$\downarrow$ \\
\midrule
\multirow{7}{*}{DNN}
& CET             & 0.27$\pm$0.16 & 0.60$\pm$0.22 & 0.86$\pm$0.09 & 378.9$\pm$84.1 & --            & --            & --            & -- \\
& AReS            & --            & --            & --            & --             & --            & --            & --            & -- \\
& GLOBE-CE        & 0.87$\pm$0.30 & 0.04$\pm$0.20 & 0.89$\pm$0.30 & 23.2$\pm$0.1   & 0.98$\pm$0.10 & 0.00$\pm$0.00 & 0.98$\pm$0.10 & 170.5$\pm$7 \\
& GLANCE          & 0.65$\pm$0.06 & $<\!$0.01$\pm$0.006 & 0.66$\pm$0.06 & 10.75$\pm$0.33 & 0.96$\pm$0.003 & 0.00$\pm$0.00 & 0.96$\pm$0.003 & 47.22$\pm$5.1 \\
& T-CREx$_{0.9}$  & --            & --            & --            & --             & 0.66$\pm$0.30 & 0.91$\pm$0.08 & 1.58$\pm$0.20 & 15.04$\pm$5.09 \\
& T-CREx$_{0.99}$ & --            & --            & --            & --             & 0.64$\pm$0.35 & 0.86$\pm$0.16 & 1.50$\pm$0.31 & 10.4$\pm$0.34 \\
& \best{SOGAR}    & \textbf{0.17$\pm$0.01} & 0.29$\pm$0.04 & \textbf{0.46$\pm$0.04} & 268.9$\pm$6.7 & \textbf{0.03$\pm$0.01} & 0.03$\pm$0.03 & \textbf{0.06$\pm$0.03} & $1{,}248.4\pm67.4$ \\
\midrule
\multirow{7}{*}{XGBoost}
& CET             & 0.51$\pm$0.11 & 0.24$\pm$0.20 & 0.75$\pm$0.06 & 438.8$\pm$65.7 & --            & --            & --            & -- \\
& AReS            & --            & --            & --            & --             & --            & --            & --            & -- \\
& GLOBE-CE        & 0.91$\pm$0.20 & \textbf{0.01$\pm$0.0} & 0.92$\pm$0.20 & 30.91$\pm$0.0 & 0.98$\pm$0.01 & \textbf{0.00$\pm$0.00} & 0.98$\pm$0.10 & 205.2$\pm$6 \\
& GLANCE          & 0.69$\pm$0.12 & 0.03$\pm$0.03 & 0.72$\pm$0.11 & 10.6$\pm$0.21  & 0.96$\pm$0.01 & $<\!$0.01$\pm$0.001 & 0.96$\pm$0.01 & 35.1$\pm$1.53 \\
& T-CREx$_{0.9}$  & 0.32$\pm$0.16 & 0.99$\pm$0.01 & 1.31$\pm$0.15 & 6.87$\pm$0.33  & 0.70$\pm$0.17 & 0.90$\pm$0.13 & 1.46$\pm$0.25 & 12.52$\pm$2.48 \\
& T-CREx$_{0.99}$ & --            & --            & --            & --             & 0.35$\pm$0.03 & 0.99$\pm$0.02 & 1.13$\pm$0.20 & 10.23$\pm$0.4 \\
& \best{SOGAR}    & \textbf{0.21$\pm$0.03} & 0.26$\pm$0.06 & \textbf{0.47$\pm$0.04} & 250.7$\pm$32.3 & \textbf{0.30$\pm$0.04} & 0.34$\pm$0.06 & \textbf{0.64$\pm$0.04} & 763.0$\pm$62.8 \\
\midrule
\multirow{7}{*}{LightGBM}
& CET             & 0.60$\pm$0.02 & 0.01$\pm$0.01 & 0.61$\pm$0.01 & 383.9$\pm$119   & --            & --            & --            & -- \\
& AReS            & 0.36$\pm$0.02 & 0.80$\pm$0.08 & 1.15$\pm$0.09 & 213.7$\pm$29.2  & --            & --            & --            & -- \\
& GLOBE-CE        & 0.92$\pm$0.30 & \textbf{$<\!$0.01$\pm$0.0} & 0.92$\pm$0.30 & 29.8$\pm$0.5 & 0.97$\pm$0.01 & \textbf{0.00$\pm$0.00} & 0.97$\pm$0.01 & 171.6$\pm$8.9 \\
& GLANCE          & 0.71$\pm$0.09 & \textbf{$<\!$0.01$\pm$0.00} & 0.71$\pm$0.09 & \textbf{11.0$\pm$0.42} & 0.93$\pm$0.01 & $<\!$0.01$\pm$0.00 & 0.93$\pm$0.01 & \textbf{32.0$\pm$0.89} \\
& T-CREx$_{0.9}$  & 0.38$\pm$0.20 & 0.66$\pm$0.26 & 1.04$\pm$0.21 & 6.83$\pm$0.55   & 0.28$\pm$0.15 & 0.96$\pm$0.04 & 1.24$\pm$0.16 & 10.88$\pm$2.52 \\
& T-CREx$_{0.99}$ & 0.47$\pm$0.14 & 0.58$\pm$0.28 & 1.04$\pm$0.23 & 8.71$\pm$15.79  & 0.45$\pm$0.17 & 0.99$\pm$0.01 & 1.44$\pm$0.17 & 10.5$\pm$0.85 \\
& \best{SOGAR}    & \textbf{0.19$\pm$0.03} & 0.06$\pm$0.02 & \textbf{0.25$\pm$0.03} & 166.2$\pm$4.3 & 0.35$\pm$0.03 & 0.17$\pm$0.02 & \textbf{0.52$\pm$0.03} & $1{,}235.3\pm75.3$ \\
\bottomrule
\end{tabular}}
\end{subtable}

\end{table}

\section{Timeout Ablation}
\label{app:timeouts}

Because SOGAR is an \emph{anytime} algorithm, a wall-clock timeout can
be imposed to enforce a hard budget, with the solver returning the best
non-dominated solutions found so far. Figure~\ref{fig:pareto_timeout}
illustrates this behaviour on the Bank dataset with LightGBM, where the
full Pareto front requires 166\,s: even at 10\,s, recovered solutions
exist below $(0.2, 0.4)$ in cost--loss space, outperforming the
second-best competing method CET (invalidity $0.61$ on Bank,
Table~\ref{tab:results_summary_b}). Table~\ref{tab:timeout} quantifies
this across all datasets, reporting degradation as the relative increase
in invalidity over the no-timeout baseline:
\[
  \text{Degradation} = \frac{\text{inv}_{\tau} -
  \text{inv}_{\infty}}{\text{inv}_{\infty}} \times 100\%.
\]

\begin{figure}[t]
    \centering
    \includegraphics[width=0.58\linewidth]{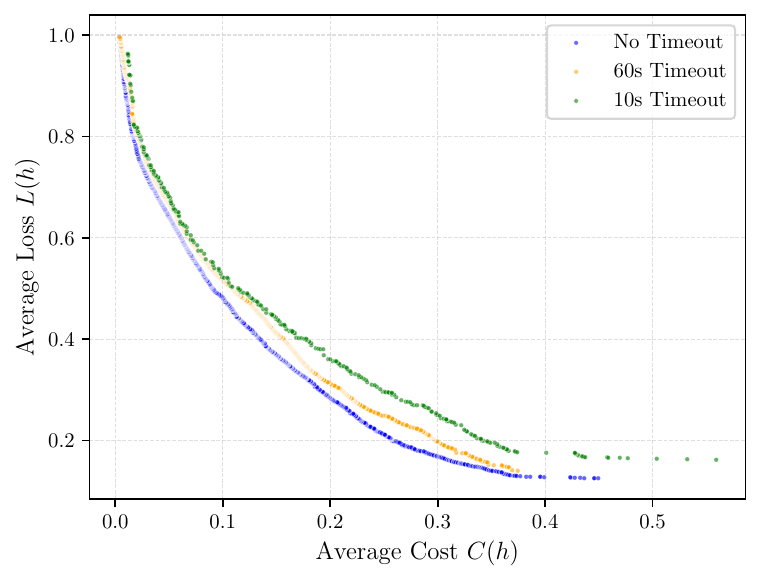}
    \caption{Solutions recovered by SOGAR on Bank (LightGBM) under
    different timeout values. Blue: full Pareto front (166\,s).
    Orange: best solutions within 60\,s. Green: best solutions within
    10\,s.}
    \label{fig:pareto_timeout}
\end{figure}

For Attrition, German, and Bank, even the 10-second budget produces degradation of only 8--16\%, while 60-second solutions are near-optimal.
The Adult dataset requires at least 200\,s to produce any valid solution, consistent with its higher baseline runtime (1{,}235\,s,
Table~\ref{tab:results_summary_b}) and the combinatorial complexity of its large, high-dimensional feature space. Even under the 200-second constraint, Adult degradation is only 5.8\%. At 200\,s, SOGAR's invalidity of $0.55$ remains below all baselines that return solution on the Adult setting, with the runners-up being the solution of GLANCE at $0.93$ average invalidity.

\begin{table}[t]
\centering
\caption{Timeout ablation (LightGBM, depth 3, 10-fold CV). Degradation is the relative invalidity increase over the no-timeout baseline.
Mean$\pm$std reported across folds.}
\label{tab:timeout}
\footnotesize
\begin{tabular}{llcccc}
\toprule
\textbf{Dataset} & \textbf{Timeout} &
\textit{cost}$\downarrow$ & \textit{loss}$\downarrow$ &
\textit{inv.}$\downarrow$ & \textbf{Deg.\ (\%)} \\
\midrule
\multirow{3}{*}{Attrition}
 & 10\,s      & $0.33\!\pm\!0.02$ & $0.10\!\pm\!0.03$ & $0.43\!\pm\!0.03$ & $+16.2$ \\
 & 60\,s      & $0.31\!\pm\!0.02$ & $0.08\!\pm\!0.02$ & $0.39\!\pm\!0.02$ & $+5.4$  \\
 & No timeout & $0.30\!\pm\!0.02$ & $0.07\!\pm\!0.02$ & $0.37\!\pm\!0.02$ & ---     \\
\midrule
\multirow{3}{*}{German}
 & 10\,s      & $0.12\!\pm\!0.00$ & $0.04\!\pm\!0.01$ & $0.16\!\pm\!0.01$ & $+14.3$ \\
 & 60\,s      & $0.11\!\pm\!0.01$ & $0.03\!\pm\!0.01$ & $0.14\!\pm\!0.01$ & $0.0$   \\
 & No timeout & $0.11\!\pm\!0.01$ & $0.03\!\pm\!0.01$ & $0.14\!\pm\!0.01$ & ---     \\
\midrule
\multirow{3}{*}{Bank}
 & 10\,s      & $0.19\!\pm\!0.05$ & $0.08\!\pm\!0.04$ & $0.28\!\pm\!0.04$ & $+12.0$ \\
 & 60\,s      & $0.19\!\pm\!0.04$ & $0.08\!\pm\!0.03$ & $0.27\!\pm\!0.04$ & $+8.0$  \\
 & No timeout & $0.19\!\pm\!0.03$ & $0.06\!\pm\!0.02$ & $0.25\!\pm\!0.03$ & ---     \\
\midrule
\multirow{4}{*}{Adult}
 & 10\,s      & \multicolumn{3}{c}{\textit{no solution found}} & ---    \\
 & 60\,s      & \multicolumn{3}{c}{\textit{no solution found}} & ---    \\
 & 200\,s     & $0.35\!\pm\!0.01$ & $0.20\!\pm\!0.04$ & $0.55\!\pm\!0.04$ & $+5.8$ \\
 & No timeout & $0.35\!\pm\!0.03$ & $0.17\!\pm\!0.02$ & $0.52\!\pm\!0.03$ & ---    \\
\bottomrule
\end{tabular}
\end{table}

\section{Ablation Studies}
\label{app:ablations}

\subsection{Tree Depth}
\label{app:ablations:depth}

SOGAR's recourse summary tree has a maximum depth hyperparameter $d$,
which directly controls the number of subgroups exposed to practitioners
(up to $2^d$ leaf nodes). We evaluate $d \in \{2, 3, 4\}$. Depth 3 is
the default used in all main-paper experiments
(Tables~\ref{tab:results_summary}); depth 2 and depth 4 isolate the
effect of tree expressiveness on solution quality and runtime.

\paragraph{Depth 2.}
Table~\ref{tab:depth2} reports results for depth $d = 2$. Across all datasets
and classifiers, depth-2 trees run substantially faster --- often by an
order of magnitude --- while the increase in invalidity relative to
depth 3 is modest. For example, on Adult with LightGBM, invalidity
rises from $0.52$ (depth 3) to $0.57$ (depth 2) while runtime drops
from $1{,}235$\,s to $51$\,s, a $24\times$ speed-up for a 9.6\%
quality degradation. Depth 2 still outperforms all baselines on the
invalidity metric across every dataset and classifier combination,
making it a practical default when compute is constrained. The
solutions are also maximally interpretable, producing exactly four
subgroups with simple branching rules.

\begin{table}[t]
\centering
\caption{Depth-2 results across datasets and classifiers (10-fold CV,
mean$\pm$std). Lower is better for all metrics.}
\label{tab:depth2}
\small
\setlength{\tabcolsep}{5pt}
\begin{tabular}{llcccc}
\toprule
\textbf{Dataset} & \textbf{Model} &
\textit{cost}$\downarrow$ & \textit{loss}$\downarrow$ &
\textit{inv.}$\downarrow$ & \textit{time(s)}$\downarrow$ \\
\midrule
\multirow{3}{*}{Attrition}
 & DNN      & $0.06\pm0.00$ & $<\!0.01\pm0.00$ & $0.06\pm0.00$ & $9\pm1$  \\
 & LightGBM & $0.31\pm0.03$ & $0.12\pm0.02$    & $0.43\pm0.03$ & $12\pm1$ \\
 & XGBoost  & $0.25\pm0.01$ & $0.06\pm0.01$    & $0.31\pm0.01$ & $10\pm0$ \\
\midrule
\multirow{3}{*}{German}
 & DNN      & $0.02\pm0.01$ & $<\!0.01\pm0.00$ & $0.03\pm0.01$ & $3\pm0$ \\
 & LightGBM & $0.11\pm0.01$ & $0.05\pm0.01$    & $0.16\pm0.01$ & $3\pm1$ \\
 & XGBoost  & $0.11\pm0.01$ & $0.04\pm0.02$    & $0.15\pm0.02$ & $3\pm0$ \\
\midrule
\multirow{3}{*}{Bank}
 & DNN      & $0.17\pm0.01$ & $0.33\pm0.03$ & $0.49\pm0.04$ & $8\pm0$ \\
 & LightGBM & $0.19\pm0.03$ & $0.08\pm0.03$ & $0.27\pm0.03$ & $6\pm1$ \\
 & XGBoost  & $0.22\pm0.04$ & $0.29\pm0.07$ & $0.51\pm0.04$ & $8\pm1$ \\
\midrule
\multirow{3}{*}{Adult}
 & DNN      & $0.03\pm0.01$ & $0.04\pm0.03$ & $0.07\pm0.04$ & $47\pm1$ \\
 & LightGBM & $0.37\pm0.02$ & $0.20\pm0.04$ & $0.57\pm0.03$ & $51\pm3$ \\
 & XGBoost  & $0.32\pm0.08$ & $0.37\pm0.11$ & $0.69\pm0.04$ & $63\pm3$ \\
\bottomrule
\end{tabular}
\end{table}

\paragraph{Depth 3 (default).}
Depth-3 results with standard deviations across 10 folds are reported
in Tables~\ref{tab:results_summary}. This depth offers the
best balance between solution quality and computational cost: it
consistently achieves the lowest invalidity among all baselines while
remaining tractable on all four datasets without timeouts.

\paragraph{Depth 4.}
Table~\ref{tab:depth4} reports results for $d = 4$ with the LightGBM classifier. At this depth, the search space grows substantially, as the solver must explore up to 16 leaf nodes, and the number of candidate subproblems increases combinatorially. On the Adult dataset, depth-4 solutions require approximately $15{,}937$\,s ($\approx$4.4 hours),
a $12.9\times$ increase over depth 3. The invalidity improvement is
marginal ($0.52 \to 0.50$), yielding diminishing returns relative to
the computational investment. On smaller datasets where depth-4 remains tractable, some additional quality is recoverable: invalidity improves by 11\% on Attrition ($0.37 \to 0.33$), 7\% on German ($0.14 \to 0.13$), and 12\% on Bank ($0.25\to 0.22$), though depth-3 already outperforms all baselines in these settings, in a significantly lower computational cost for a complete Pareto front of solutions. This confirms that depth 3 represents the practical sweet spot for the datasets considered: deeper trees offer only incremental quality gains while incurring substantially higher costs, consistent with the NP-hard nature of the underlying optimization problem. 

Table~\ref{tab:depth_summary} consolidates the depth trade-off for the
Adult dataset with LightGBM, which is the most computationally demanding
configuration.

\begin{table}[t]
\centering
\caption{Depth-4 results with LightGBM across datasets (10-fold CV,
mean$\pm$std). Lower is better for all metrics.}
\label{tab:depth4}
\small
\setlength{\tabcolsep}{5pt}
\begin{tabular}{lcccc}
\toprule
\textbf{Dataset} &
\textit{cost}$\downarrow$ & \textit{loss}$\downarrow$ &
\textit{inv.}$\downarrow$ & \textit{time(s)}$\downarrow$ \\
\midrule
Attrition & $0.29\pm0.02$ & $0.03\pm0.01$ & $0.33\pm0.02$ & $2{,}710\pm97$s \\
German    & $0.11 \pm 0.01$ & $0.02\pm 0.01$ & $0.13\pm0.01$ & $582.5\pm51$s \\
Bank      & $0.17 \pm 0.04$ & $0.05 \pm 0.02$ & $0.22 \pm 0.03$ & $3{,}112 \pm 478$s \\
Adult     & $0.33\pm0.02$ & $0.17\pm0.02$ & $0.50\pm0.03$ & $15{,}937\pm1,410$s \\
\bottomrule
\end{tabular}
\end{table}

\begin{table}[t]
\centering
\caption{Depth ablation summary on Adult Income (LightGBM, 10-fold CV).
Degradation is relative invalidity change vs.\ depth 3.}
\label{tab:depth_summary}
\footnotesize
\begin{tabular}{lcccc}
\toprule
\textbf{Depth} &
\textit{inv.}$\downarrow$ & \textit{time(s)}$\downarrow$ &
\textbf{Deg.\ (\%)} & \textbf{Speed-up} \\
\midrule
2 & $0.57\pm0.03$ & $51\pm3$          & $+9.6\%$ & $24.2\times$ \\
3 & $0.52\pm0.03$ & $1{,}235\pm75$    & ---       & $1.0\times$  \\
4 & $0.50\pm0.03$ & $15{,}937\pm1{,}410$ & $-3.8\%$ & $0.08\times$ \\
\bottomrule
\end{tabular}
\end{table}

\subsection{Action Sparsity}
\label{app:ablations:sparsity}

In the default configuration SOGAR allows actions that modify up to
three features simultaneously (action sparsity $k=3$). We assess a
sparser setting ($k=2$), evaluating the trade-off between solution
quality, runtime, and memory usage (Table~\ref{tab:sparsity}).
Experiments use the LightGBM classifier at depth 3.

\begin{table}[t]
\centering
\caption{Action-sparsity ablation: $k=2$ vs.\ default $k=3$
(LightGBM, depth 3, 10-fold CV, mean$\pm$std).}
\label{tab:sparsity}
\footnotesize
\begin{tabular}{lcccc}
\toprule
\textbf{Dataset} &
\textit{cost}$\downarrow$ & \textit{loss}$\downarrow$ &
\textit{inv.}$\downarrow$ & \textit{time(s)}$\downarrow$ \\
\midrule
Attrition & $0.32\pm0.02$ & $0.17\pm0.03$ & $0.50\pm0.03$ & $162.6\pm6.8$  \\
German    & $0.12\pm0.00$ & $0.08\pm0.01$ & $0.19\pm0.01$ & $23.4\pm1.7$   \\
Bank      & $0.14\pm0.03$ & $0.22\pm0.03$ & $0.36\pm0.04$ & $116.7\pm2.4$  \\
Adult     & $0.40\pm0.16$ & $0.44\pm0.18$ & $0.84\pm0.04$ & $533.9\pm25.8$ \\
\bottomrule
\end{tabular}
\end{table}

Restricting actions to two features yields approximately a 50\%
reduction in runtime (e.g.\ Attrition: $356\to163$\,s; Adult:
$1{,}235\to534$\,s) and a proportional reduction in action-cache
memory, since the cache scales with the number of distinct actions.
The cost metric is largely unaffected, but recourse loss increases
noticeably --- particularly on Adult ($0.17\to0.44$) and Attrition
($0.07\to0.17$) --- because sparser actions can no longer
simultaneously address all necessary feature changes for harder
subgroups. On Adult, the invalidity rises from $0.52$ to $0.84$,
which is still below T-CREx ($1.24$), and GLANCE ($0.93$),
illustrating that sparsity constraints interact strongly with dataset
complexity. We therefore recommend $k=2$ only when memory or time is
severely constrained, and note that even under this restriction SOGAR
remains competitive with baselines on the smaller datasets.

\subsection{Number of Discretisation Bins}
\label{app:ablations:bins}

Continuous features are discretised into bins before the tree search.
In the default configuration the number of bins varies per feature
according to its empirical distribution (ranging from 10 to 50 bins).
We evaluate two reduced-bin settings (5 and 10 bins applied uniformly
to all numerical features) on the Attrition dataset with the LightGBM
classifier (Table~\ref{tab:bins}). Degradation is computed relative to
the default-bin baseline.

\begin{table}[h]
\centering
\caption{Binning ablation (LightGBM, Attrition, depth 3, 10-fold CV,
mean$\pm$std). Degradation is relative invalidity increase vs.\ default
bins.}
\label{tab:bins}
\small
\begin{tabular}{lccccc}
\toprule
\textbf{Bins} &
\textit{cost}$\downarrow$ & \textit{loss}$\downarrow$ &
\textit{inv.}$\downarrow$ & \textit{time(s)}$\downarrow$ &
\textbf{Deg.\ (\%)} \\
\midrule
Default (10--50)  & $0.30\pm0.02$ & $0.07\pm0.02$ & $0.37\pm0.02$ & $356.4\pm9.4$ & ---      \\
10 bins (uniform) & $0.31\pm0.01$ & $0.08\pm0.02$ & $0.38\pm0.01$ & $163\pm4$         & $+3.2\%$ \\
5 bins (uniform)  & $0.32\pm0.01$ & $0.07\pm0.01$ & $0.39\pm0.01$ & $127\pm4$         & $+4.7\%$ \\
\bottomrule
\end{tabular}
\end{table}

Reducing the bin count from the default distribution-adaptive scheme to a uniform 10-bin or 5-bin encoding roughly halves runtime (to 163\,s and 127\,s respectively) with only 3--5\% degradation in invalidity.
On this configuration, SOGAR  with 5 uniform bins still outperforms all baselines on the invalidity metric, suggesting that reduced binning is a viable strategy when runtime or memory is constrained.
The runtime reduction stems from two sources: fewer bins produce fewer binarized features, shrinking the split search
space, and they also reduce the number of distinct actions in the cache.
We note that very coarse binning risks conflating meaningfully different feature values within a single bin. A practical lower bound are 5 bins below which solution quality may degrade more sharply for features with heavy-tailed distributions.

\section{Leaf solution acceleration}
\label{app:hardware_acceleration}

We have extended STreeD to include multi-threaded CPU and GPU acceleration in our optimization task.
As mentioned in Section \ref{sec:proposed_method}, when assigning an action to a leaf, the solver needs to try all actions and for each of them aggregate the cost-loss pairs for that action of the individuals reaching that leaf.
This is an embarrassingly parallelizable task which we implement using the TBB framework for CPU parallelization and the CUDA framework for GPU parallelization.
Users can make use whichever acceleration their hardware supports.
Table \ref{tab:runtime_comp} presents the runtimes all datasets using the single-threaded implementation and the GPU implementation.
We can see that the GPU incurs significant speed-ups, eliminating time-outs and in the case of the German Credit dataset the running time is reduced by a factor of over 23.

\begin{table}[h]
\centering
\caption{Comparison of SOGAR's running times with and without GPU acceleration on the all datasets. The timeout is set to 1 hour.}
\label{tab:runtime_comp}
\small
\begin{tabular}{lrr}
\toprule
\textbf{Dataset} & \textbf{Single-Threaded} & \textbf{GPU Acceleration} \\
\midrule
Attrition & 708.6s & 250.0s\\
German    & 668.2s & 28.2s \\
Bank      & $>$1h & 268.9s \\
Adult     & $>$1h & 1248.4s \\
\bottomrule
\end{tabular}
\end{table}

\section{Bias Auditing Experiment --- Extended Analysis} 
\label{app:bias_audit_appendix}

\subsection*{Setup}
\label{subsec:setup_bias}

SOGAR was ran on the Adult Income dataset~\cite{adult_2} using the default setting (32,561 training instances;  3,257 held-out, LightGBM, depth 3, MaxNodes 7, MinLeaf 500), the results of which are depicted in  Table~\ref{tab:results_summary}. Per-group metrics are computed by routing all 19,199 
unfavorably-predicted training instances through every Pareto-front solution tree, assigning each instance the cost and loss of its landed leaf, averaging per group, then plotting across the full front.

\subsection*{Classifier Fairness}
 
The LightGBM classifier exhibits measurable bias at the 
prediction level under two standard criteria, as reported 
in Table~\ref{tab:bias_audit_demo}. Under \textbf{demographic parity}, the classifier assigns 
adverse predictions ($\hat{Y}\!=\!0$) to 86.5\% of females 
versus 55.1\% of males on train (85.6\% vs.\ 55.5\% on 
test), yielding a Disparate Impact Ratio of 0.64 (test: 
0.65) --- well below the $4/5$ 
threshold~\citep{watkins2024four}.
Under \textbf{equality of opportunity}, among individuals who truly earn $>$50K ($Y\!=\!1$), the classifier correctly identifies 90.4\% of males as favorable but only 79.2\% of females, which depicts a true positive rate (TPR) gap of $+0.112$
on train ($+0.118$ on test). A female who deserves a favorable prediction faces a substantially higher probability of being wrongly denied it than an equivalent male. 
Both violations are consistent across train and test 
splits, confirming they reflect the model's learned 
structure.
 
\begin{table}[t!]
\centering
\caption{Classifier fairness metrics (LightGBM, Adult Income,
fold~1).}
\label{tab:bias_audit_demo}
\footnotesize
\setlength{\tabcolsep}{4pt}
\renewcommand{\arraystretch}{0.82}
\setlength{\aboverulesep}{0.25ex}

\setlength{\belowrulesep}{0.25ex}
\begin{tabular}{lcc}
\toprule
\textbf{Metric} & \textbf{Train} & \textbf{Test} \\
\midrule
Accuracy & 0.841 & 0.832 \\
\midrule
\multicolumn{3}{l}{\textit{Demographic parity}} \\
Female $n$ / adverse ($\hat{Y}\!=\!0$) & 9,730 / 8,417 & 1,041 / 891 \\
Male $n$ / adverse ($\hat{Y}\!=\!0$)   & 19,574 / 10,782 & 2,216 / 1,229 \\
Female adverse rate ($\hat{Y}\!=\!0$) & 86.5\% & 85.6\% \\
Male adverse rate ($\hat{Y}\!=\!0$)   & 55.1\% & 55.5\% \\
Disparate Impact Ratio                & \textbf{0.64} & \textbf{0.65} \\
\midrule
\multicolumn{3}{l}{\textit{Equality of opportunity} 
    $P(\hat{Y}\!=\!1 \mid Y\!=\!1)$} \\
Female $n$ ($Y\!=\!1$)          & 1,054 & 125 \\
Male $n$ ($Y\!=\!1$)            & 6,002 & 660 \\
Female true positive rate (TPR) & 0.792 & 0.776 \\
Male true positive rate (TPR)   & 0.904 & 0.894 \\
Gap (Male $-$ Female)           & $+0.112$ & $+0.118$ \\
\bottomrule
\end{tabular}
\end{table}
 
\subsection*{Per-Group Recourse Quality \& Disparity}
 Across all 4,142 Pareto-front solutions, females face consistently worse recourse than males on every metric, 
yet the nature of the disadvantage is asymmetric. The invalidity gap --- $+0.093$ in train data and $+0.081$ in test data --- is driven mostly by loss ($+0.070$) rather than cost ($+0.023$). Note also that the scale of the plots is different as loss reaches 1.0 at the rightmost, while cost reaches a bit above 0.5, so the loss gap can appear less exaggerated than it actually is.
The full per-group breakdown is reported in Table~\ref{tab:bias_audit_app}, with train--test consistency confirming the disparity reflects the classifier's learned structure rather than properties of the recourse solutions.

The two components also behave differently along the front, as shown in Figure~\ref{fig:loss_cost_disparity}. 
The invalidity gap opens wide early, compresses sharply around $I \approx 0.58$--$0.62$ where the female and male curves briefly converge, then widens significantly as the invalidity increases very rapidly for females.
Eventually, invalidity rises for males as well and meets the female invalidity, depicting that in both extremes invalidity marginally converges.
However, it is this late divergence in high cost region of $C \approx 0.62$--$0.9$ that sustains the near-universal disparity across the front.

\begin{table}[t!]
\centering
\caption{Per-group recourse quality across the Pareto front 
(4,142 solutions). Standard deviation is across solutions. 
Gap~$=$~Female~$-$~Male per solution, averaged across the Pareto front.}
\label{tab:bias_audit_app}
\small
\setlength{\tabcolsep}{4pt}
\renewcommand{\arraystretch}{0.82}
\setlength{\aboverulesep}{0.25ex}
\begin{tabular}{llcc}
\toprule
\textbf{Metric} & \textbf{Group / Statistic} 
    & \textbf{Train} & \textbf{Test} \\
\midrule
\multirow{2}{*}{Avg.\ cost}
    & Female & $0.116 \pm 0.143$ & $0.116 \pm 0.145$ \\
    & Male   & $0.094 \pm 0.112$ & $0.097 \pm 0.117$ \\
\multirow{2}{*}{Avg.\ loss}
    & Female & $0.672 \pm 0.285$ & $0.676 \pm 0.286$ \\
    & Male   & $0.602 \pm 0.231$ & $0.614 \pm 0.239$ \\
\multirow{2}{*}{Avg.\ invalidity}
    & Female & $0.788 \pm 0.151$ & $0.792 \pm 0.151$ \\
    & Male   & $0.696 \pm 0.133$ & $0.711 \pm 0.134$ \\
\midrule
\multirow{3}{*}{Mean gap (F$-$M)}
    & Cost       & $+0.023$ & $+0.019$ \\
    & Loss       & $+0.070$ & $+0.062$ \\
    & Invalidity & $+0.093$ & $+0.081$ \\
\bottomrule
\end{tabular}
\end{table}

\subsection*{Pareto Front}

SOGAR audits bias as a by-product of its core output, as it returns a Pareto front of recourse summary trees.
An auditor can route all instances through every solution and compare per-group cost, loss, and invalidity at each operating point. 
A disparity that holds across the entire front, as the 96.4\% worse invalidity for females shows in Table~\ref{tab:bias_audit_app}, cannot be attributed to a particular cost-loss trade-off choice and instead reflects the classifier's decision boundary itself. This front-wide diagnostic is unavailable to methods that return a single recourse summary or a fixed set of counterfactuals.

\end{document}